\documentclass{article}

\usepackage[nonatbib,preprint]{neurips_2025}

\usepackage{multirow}
\usepackage{array}
\usepackage{pifont}
\usepackage{arydshln}
\usepackage{tcolorbox}
\usepackage{wrapfig}
\usepackage[utf8]{inputenc} 
\usepackage{needspace}
\usepackage{float}
\usepackage{enumitem}
\usepackage{lipsum}
\usepackage[caption = true]{subfig}
\usepackage[T1]{fontenc}
\usepackage{hyperref}       
\usepackage{url}            
\usepackage{booktabs}       
\usepackage{amsfonts}       
\usepackage{nicefrac}       
\usepackage{microtype}      
\usepackage{graphicx}
\usepackage{amsmath}
\usepackage{cleveref}
\usepackage{amsthm}
\usepackage{verbatim}
\usepackage{listings}
\usepackage{todonotes}

\usepackage{amssymb}
\hypersetup{colorlinks,linkcolor={red},citecolor={blue},urlcolor={red}} 
\usepackage{amsthm}

\newtheorem{theorem}{Theorem}[]
\newtheorem{corollary}{Corollary}[theorem]
\newtheorem{lemma}[theorem]{Lemma}
\newtheorem{proposition}[theorem]{Proposition}

\definecolor{codegreen}{rgb}{0,0.6,0}
\definecolor{codegray}{rgb}{0.5,0.5,0.5}
\definecolor{codepurple}{rgb}{0.58,0,0.82}
\definecolor{backcolour}{rgb}{0.95,0.95,0.92}

\usepackage{listings}
\usepackage{color}
\usepackage{hyperref}

\definecolor{dkgreen}{rgb}{0,0.6,0}
\definecolor{gray}{rgb}{0.5,0.5,0.5}
\definecolor{mauve}{rgb}{0.58,0,0.82}

\graphicspath{{../}}

\title{ Rethinking Language Model Scaling under Transferable Hypersphere Optimization} 

\author{
\textbf{Liliang Ren}$^{1}$ \quad \textbf{Yang Liu}$^{1}$ \quad 
\textbf{Yelong Shen}$^{1}$ \quad \textbf{Weizhu Chen}$^{1}$  \vspace{2mm}\\
$^{1}$Microsoft \\
\texttt{renll1402@gmail.com \quad wzchen@microsoft.com}
}

\begin{document}

\maketitle

\begin{abstract}

Scaling laws for large language models depend critically on the optimizer and parameterization. Existing hyperparameter transfer laws are mainly developed for first-order optimizers, and they do not structurally prevent training instability at scale. Recent hypersphere optimization methods constrain weight matrices to a fixed-norm hypersphere, offering a promising alternative for more stable scaling.
We introduce \textbf{HyperP} (Hypersphere Parameterization), the first framework for transferring optimal learning rates across model width, depth, training tokens, and Mixture-of-Experts (MoE) granularity under the Frobenius-sphere constraint with the Muon optimizer. We prove that weight decay is a first-order no-op on the Frobenius sphere, show that Depth-$\mu$P remains necessary, and find that the optimal learning rate follows the same data-scaling power law with the ``magic exponent'' 0.32 previously observed for AdamW. A single base learning rate tuned at the smallest scale transfers across all compute budgets under HyperP, yielding $1.58\times$ compute efficiency over a strong Muon baseline at $6\times10^{21}$ FLOPs. Moreover, HyperP delivers transferable stability: all monitored instability indicators, including $Z$-values, output RMS, and activation outliers, remain bounded and non-increasing under training FLOPs scaling.
We also propose SqrtGate, an MoE gating mechanism derived from the hypersphere constraint that preserves output RMS across MoE granularities for improved granularity scaling, and show that hypersphere optimization enables substantially larger auxiliary load-balancing weights, yielding both strong performance and good expert balance. We release our training codebase at \url{https://github.com/microsoft/ArchScale}.

\end{abstract}

\section{Introduction}

Neural scaling laws \cite{kaplan2020scaling, chinchilla, team2025every} are central to the compute-efficient development of Large Language Models (LLMs) \cite{gpt2, gpt3,openai2023gpt4, comanici2025gemini,deepseek-ai2024deepseekv3,team2025kimi,yang2025qwen3}. In practice, architectural designs and data recipes are explored at small scales for cost-savings, hoping that the improvements will persist when scaled up to prohibitively expensive compute budgets. However, identifying the true scaling behavior requires each model along the curve to be trained with near-optimal hyperparameters for its own scale. Even with well-tuned hyperparameters, scaling up training FLOPs routinely triggers logit explosion, activation outliers, and loss spikes \cite{zoph2022st0moe0, dettmers2022llm0int8000, chowdhery2023palm, qiu2025gated, qiu2026unified} that push training off the optimal trajectory or even lead to training failures \cite{zhang2022opt0, ren2025decoderhybriddecoder}. Existing hyperparameter transfer frameworks \cite{yang2022tensor, yang2023tensor, bjorck2025scaling, li2025predictable, chen2025how, mlodozeniec2025completedhyperparametertransfermodules} primarily study the transfer of optimal learning rate (LR), weight decay, and batch size across model width and depth, while largely overlooking how the learning rate should scale with training tokens to remain transferable across training FLOPs under second-order optimizers. Moreover, they do not provide structural guarantees on training stability at larger scales. In practice, mitigations such as $Z$-loss regularization \cite{zoph2022st0moe0} and careful weight decay scheduling \cite{defazio2025gradients} are applied as ad hoc patches rather than principled solutions. Recent work on hypersphere optimization \cite{bernstein2025manifolds, wen2025hyperball_part1, xie2026controlled} offers a fundamentally different approach. By constraining weight matrices to a fixed-norm sphere, hypersphere optimization provides structural stability guarantees: the weight-norm constraint naturally bounds output logit magnitudes for each linear projection. It also has the potential to eliminate weight decay, a notorious hyperparameter whose optimal value depends intricately on learning rate, training duration \cite{wang2024set, bergsma2025power}, and model width \cite{chen2025how}.

In this work, we derive the first learning rate transfer laws across training FLOPs for hypersphere optimization, covering width, depth, training duration, and MoE granularity, under a typical second-order hypersphere optimizer MuonH \cite{wen2025hyperball_part1}. We summarize our transfer laws as HyperP (Hypersphere Parameterization), a framework under which a single base learning rate tuned at the smallest scale transfers optimally to all compute budgets. Our theoretical and empirical results reveal that hypersphere optimization, when equipped with proper transfer laws, achieves not only optimal scaling efficiency but also \emph{transferable stability}: the same hyperparameters that work at small scale produce equally or more stable training dynamics at large scale. With these results, fair comparisons of architectural scaling become possible: every model at every scale is trained at its transferred optimal learning rates, so the resulting scaling curves reflect each architecture's near-optimal performance. Our contributions are as follows:

\begin{itemize}[leftmargin=1.5em,itemsep=2pt]

    \item \textbf{First transfer laws across FLOPs for hypersphere optimization.} We derive HyperP, which achieves optimal LR transfer across training FLOPs, spanning width, depth, training tokens, and MoE granularity. We prove that weight decay is a first-order no-op under Frobenius-sphere optimization, and empirically show that removing weight decay does not harm model quality. We also derive that Depth-$\mu$P \cite{yang2023tensor} is still required, disproving the claim that MuonH is inherently depth-transferable \cite{wen2025hyperball_part1}, and discover the ``magic exponent'' 0.32 for data scaling, matching the previous result on AdamW \cite{bjorck2025scaling} and suggesting universality across optimizers.

    \item \textbf{Transferable stability.} Empirically, we show that HyperP yields \emph{stability transfer}: all six monitored instability indicators ($Z$-values, output RMS, activation outlier percentages for both attention and MoE sub-layers) are bounded and non-increasing as we scale training FLOPs for MoE models from 913M to 13.3B total parameters.

    \item \textbf{Robust MoE scaling and load balancing.} We derive SqrtGate, a square-root gating mechanism that preserves output RMS across MoE granularities, reducing router $Z$-value peaks by $5\times$ compared to standard gating. HyperP's optimal LR transfers robustly across MoE sparsity ($S \in \{1,\ldots,32\}$) and granularity ($k \in \{2,\ldots,64\}$). It also allows for much larger auxiliary load balancing weights, achieving the best validation loss and expert balance simultaneously, in contrast to prior findings that a large load balancing weight hurts model quality \cite{lin2024auxloss}.
    
    \item \textbf{Scalable compute efficiency leverage.}
     A single base LR tuned at the smallest scale with 208M active parameters transfers effectively to all compute budgets explored. At the largest $6\times 10^{21}$ FLOPs, HyperP achieves $1.58\times$ Compute Efficiency Leverage (CEL) \cite{kaplan2020scaling,team2025every} over a strong Muon baseline with $\mu$P++ and weight decay scaling for dense models, and our MoE model with $S=8, k=8$ further achieves $3.38\times$ CEL over the dense models. The advantage of HyperP over the baseline grows monotonically with scale, implying even larger gains at frontier compute. 

    \item \textbf{Architecture comparison under optimal scaling.} With HyperP, we systematically examine dense (QK-Norm \cite{pmlr-v202-dehghani23a}, Gated Attention \cite{qiu2025gated}) and MoE (SqrtGate, shared expert \cite{Dai2024DeepSeekMoETU}) architectures at their optimal performance, revealing that while the performance gains of SqrtGate and Gated Attention shrink as the training FLOPs increase, they provide significant stability gains that can remove the RMS spikes and control the exploding $Z$-values.
\end{itemize}

\section{Background}\label{sec:background}

\paragraph{Hypersphere optimization.}
Hypersphere optimization constrains weight matrices to lie on a unit sphere under a chosen matrix norm. After each gradient update, the weight matrix $W$ is re-projected as follows,
\begin{equation}
W \leftarrow C\frac{W - \eta \, G}{\|W - \eta \, G\|},
\label{eq:hyper_opt}
\end{equation}
where $\eta$ is learning rate, $C$ is a constant, $G$ denotes the weight update and $\|\cdot\|$ is the chosen matrix norm. Several choices of the matrix norm have been proposed recently: MuonH \cite{wen2025hyperball_part1} uses the Frobenius norm with the Muon optimizer \cite{jordan2024muon}; Both MuonSphere \cite{xie2026controlled}  and SSO \cite{xie2026controlled}  use the spectral norm, while SSO further applies the steepest descent on the spectral sphere. Previous works have explored column- and row-wise weight normalization \cite{salimans2016weight,karras2023analyzing,loshchilov2025ngpt, fu2025nemotron0flash0}, while in this work we focus primarily on matrix-wise normalization for theoretical simplicity.

\paragraph{MuonH optimizer.}
MuonH (Muon-Hyperball) \cite{wen2025hyperball_part1} instantiates hypersphere optimization with the Muon optimizer \cite{jordan2024muon} and Frobenius norm, normalizing both the weight and the update:
\begin{equation}
\widehat G = c_G \frac{G}{\|G\|_F}, \qquad W^{+} = c_W \frac{W - \eta \, \widehat G}{\|W - \eta \, \widehat G\|_F},
\label{eq:muonh_update}
\end{equation}
where $c_W = \|W_0\|_F$ is the initial weight norm, $c_G=c_W$, and $G$ is the standard Muon update. MuonH is applied to each hidden weight matrix, while AdamW \cite{adw} with the same weight and update normalization schemes (denoted as AdamH) is used for the weight matrix of the language-model head, and the remaining vector parameters and embeddings are optimized with the original AdamW.  The update normalization further ensures that the relative update magnitude $\|\Delta W\|_F / \|W\|_F$ is constant for a given layer, enabling the same learning-rate scale to be used for both MuonH and AdamH.


\section{Scaling Hypersphere Optimization}\label{sec:theory}

We first motivate our choice of the Frobenius hypersphere by showing that it can eliminate the need for weight decay under the first-order approximation. We then derive the hyperparameter transfer laws for width and depth scaling by examining the theoretical implications of hypersphere optimization. The learning rate transfer law is further established for the data scaling scenario through empirical studies. We present our Hypersphere Parameterization, \emph{HyperP}, for training FLOPs scaling by summarizing our transfer laws over width, depth, and data scaling in \Cref{tab:mup-families}. Finally, we illustrate the theoretical stability benefits of hypersphere optimization in \Cref{sec:zloss} and propose our SqrtGate mechanism for the granularity scaling of Mixture-of-Experts (MoE) in \Cref{sec:moe-theory}.

\subsection{Elimination of Weight Decay}\label{sec:wd}

Among various choices of norms in hypersphere optimization, the Frobenius norm admits a simple geometric interpretation: after projection back to a fixed Frobenius sphere, only the tangent component of an update survives to first order.

\begin{theorem}[First-order form of Frobenius-sphere updates]
\label{thm:fnorm_tangent_update}
Let $W\in\mathbb{R}^{d_{\mathrm{out}}\times d_{\mathrm{in}}}$ satisfy $\|W\|_F=c_W$, and define
\begin{equation}
\widetilde W = W + \Delta,
\qquad
W^+ = c_W \frac{\widetilde W}{\|\widetilde W\|_F}.
\label{eq:fnorm_projected_update_main}
\end{equation}
Then, for sufficiently small $\|\Delta\|_F$,
\begin{equation}
W^+ - W
=
\Pi_T(\Delta) + O(\|\Delta\|_F^2),
\label{eq:fnorm_tangent_expansion_main}
\end{equation}
where
$
\Pi_T(\Delta)
=
\Delta - \frac{\langle \Delta, W\rangle_F}{\|W\|_F^2}W
$
is the tangent-space projection at $W$.
\end{theorem}

A direct corollary is that radial shrinkage is removed to first order.

\begin{corollary}[Weight decay is a first-order no-op]
\label{cor:wd_noop}
If
$
\Delta = -\eta G - \eta \lambda W,
$
then
$
W^+ - W
=
-\eta \Pi_T(G) + O(\eta^2).
$
Hence, the weight decay term has no first-order effect under Frobenius renormalization.
\end{corollary}

The detailed proof is included in Appendix~\ref{app:proof_fnorm_tangent_update}.  Therefore, in this work, we only study the MuonH optimizer, which is based on the Frobenius norm, and we leave the formal uniqueness characterization as a future work.
Our theorem eliminates weight decay as a hyperparameter entirely, reducing the search space from the joint $(\eta, \lambda)$ plane to a single dimension $\eta$. In contrast, standard optimizers (e.g. AdamW, Muon) require careful co-tuning of the learning rate and weight decay, and the optimal weight decay interacts with the learning rate, training duration \cite{wang2024set,bergsma2025power}, and even model width \cite{chen2025how}, which further complicates hyperparameter transfer laws.

\subsection{Width Scaling}\label{sec:width}

Our derivation of the width transfer laws is based on the following observation: For $W \in \mathbb{R}^{d_{\mathrm{out}}\times d_{\mathrm{in}}}$, the spectral and Frobenius norms satisfy
\[
\|W\|_2 \le \|W\|_F \le \sqrt{r}\,\|W\|_2,
\qquad
r := \operatorname{rank}(W) \le \min(d_{\mathrm{in}},d_{\mathrm{out}}).
\]
The upper bound is attained if and only if $W$ has full rank and all its nonzero singular values are equal. When the update $\Delta W$ is orthogonalized, as in Muon, the resulting dynamics tend to avoid highly anisotropic spectra and instead favor a relatively flat singular-value profile. In that regime, $\|W\|_F/\sqrt{r}$ becomes a good proxy of $\|W\|_2$, which leads to the same width-transfer property as in $\mu$P \cite{yang2022tensor}. 

\begin{theorem}[Width transfer under Frobenius sphere]
\label{thm:width-scaling-hyperball}
Let $Y = WX$, $W \in \mathbb{R}^{d_{\mathrm{out}}\times d_{\mathrm{in}}}$, and assume
$\|W\|_{\mathrm{rms}} = C/\sqrt{d_{\mathrm{in}}}$,
equivalently $\|W\|_F = C\sqrt{d_{\mathrm{out}}}$,
for a width-independent constant $C=O(1)$.
Assume further that $W$ is approximately isotropic on its input space, in the sense that its nonzero singular values are sufficiently uniform. Equivalently, for typical inputs $X$,
\[
\|WX\|_2 \approx \frac{\|W\|_F}{\sqrt{\min(d_{\mathrm{in}},d_{\mathrm{out}})}}\,\|X\|_2.
\]
Since $\sqrt{d_{\mathrm{in}}/ \min(d_{\mathrm{in}},d_{\mathrm{out}})} = O(1)$, then
\[
\|Y\|_{\mathrm{rms}}
\approx
C\,\|X\|_{\mathrm{rms}}.
\]
\end{theorem}

The proof is included in Appendix~\ref{app:width-scaling-proof}. Therefore, hypersphere optimization with $\|W\|_F = C\sqrt{d_{\mathrm{out}}}$ preserves width transfer without explicit $1/w$ learning rate scaling as in standard $\mu$P .


\subsection{Depth Scaling}\label{sec:depth-theory}

We analyze how the learning rate and residual scaler depend on depth when optimization is performed on a Frobenius sphere. Consider a depth-\(L\) residual network
\begin{equation}
x_{l+1} = x_l + \alpha_L f_l(x_l;W_l), \qquad l=1,\dots,L,
\label{eq:residual_depth_fnorm}
\end{equation}
where \(\alpha_L\) denotes the depth-dependent residual scaler. Each matrix parameter is constrained to satisfy \(\|W_l\|_F = c_W\). We first study the hypersphere optimization with only weight normalization as in \Cref{eq:hyper_opt} and then consider the variant in which the raw update is normalized to have a fixed Frobenius norm as in \Cref{eq:muonh_update}, where $c_W$ is not necessarily equal to $c_G$.

\begin{theorem}[Depth scaling under Frobenius-sphere optimization]
\label{thm:depth_scaling_fnorm}
Under residual networks, assume that local Jacobians are $O(1)$ under the chosen width parameterization, and that the update step size is sufficiently small for first-order linearization to hold.
\begin{enumerate}
    \item \textbf{Weight renormalization under scale-dependent optimizers.} If only the weights are renormalized as in \eqref{eq:hyper_opt}, and the layerwise update satisfies $\|G_l\|_F = O(\alpha_L)$, then the total first-order function perturbation is $O(L\alpha_L^2 \eta_l)$. In particular, for the standard depth-stabilizing scaling $\alpha_L=L^{-1/2}$, a depth-independent learning rate $\eta_l=O(1)$ yields an $O(1)$ function-space update.
    
    \item \textbf{Normalizing both weight and update.} If the weight update is normalized so that $\|U_l\|_F=c_G$, then the total first-order function perturbation is $O(L\alpha_L \eta_l)$. Hence, the learning rate must scale as
    \begin{equation}
    \eta_l = O\!\left(\frac{1}{L\alpha_L}\right)
    \label{eq:main_depth_rule_short}
    \end{equation}
    to maintain an $O(1)$ function-space update. In particular, when $\alpha_L=L^{-1/2}$,
    $
    \eta_l = O(L^{-1/2}).
    $
    \item \textbf{Post-norm Architecture.} The same exponent holds for post-norm residual blocks
    \begin{equation}
    x_{l+1} = \mathrm{LayerNorm}\!\left(x_l + \alpha_L f_l(x_l;W_l)\right),
    \end{equation}
    since the standard deviation of the input of LayerNorm \cite{ba2016layer} is $O(1)$ with depth.
\end{enumerate}
\end{theorem}

The proofs are included in Appendix~\ref{app:proof_depth_scaling_fnorm}. Our theorem shows that the original forms of Depth-$\mu$P \cite{yang2023tensor} for both scale-dependent and scale-invariant optimizers are preserved under hypersphere optimization, and the post-norm residual does not remove the dependence on model depth for accumulated weight perturbations. Importantly, our result does not rely on any independence assumptions across layers. Under the small-step condition \(\eta\|G_l\|_F \ll \|W_l\|_F\), we use a first-order Taylor expansion to express the total perturbation as a sum of layerwise contributions and then apply the triangle inequality to obtain a deterministic worst-case bound. Our theorem shows that, contrary to the original claim of the authors in MuonH \cite{wen2025hyperball_part1}, the optimizer is not inherently transferable across model depth because they neglect the cumulative angular drift introduced by the summation of residual connections. We further provide empirical verification of our theoretical result in \Cref{sec:param-scaling}.

\subsection{Data Scaling}\label{sec:data}
Since there is no clear theory on how the learning rate should scale with the training tokens, we study the transfer law with empirical studies. We fix model depth $d{=}8$ (208M parameters) and vary training tokens from 10.4B to 166.4B, sweeping LR on a fine grid $\{0.004, 0.006, 0.008, 0.010, 0.012, 0.014, 0.016, 0.018\}$ with quadratic fitting in $\log(\eta)$ space. The detailed setup is provided in beginning of \Cref{sec:experiments}.

\begin{figure}[H]
    \centering
    \includegraphics[width=\linewidth]{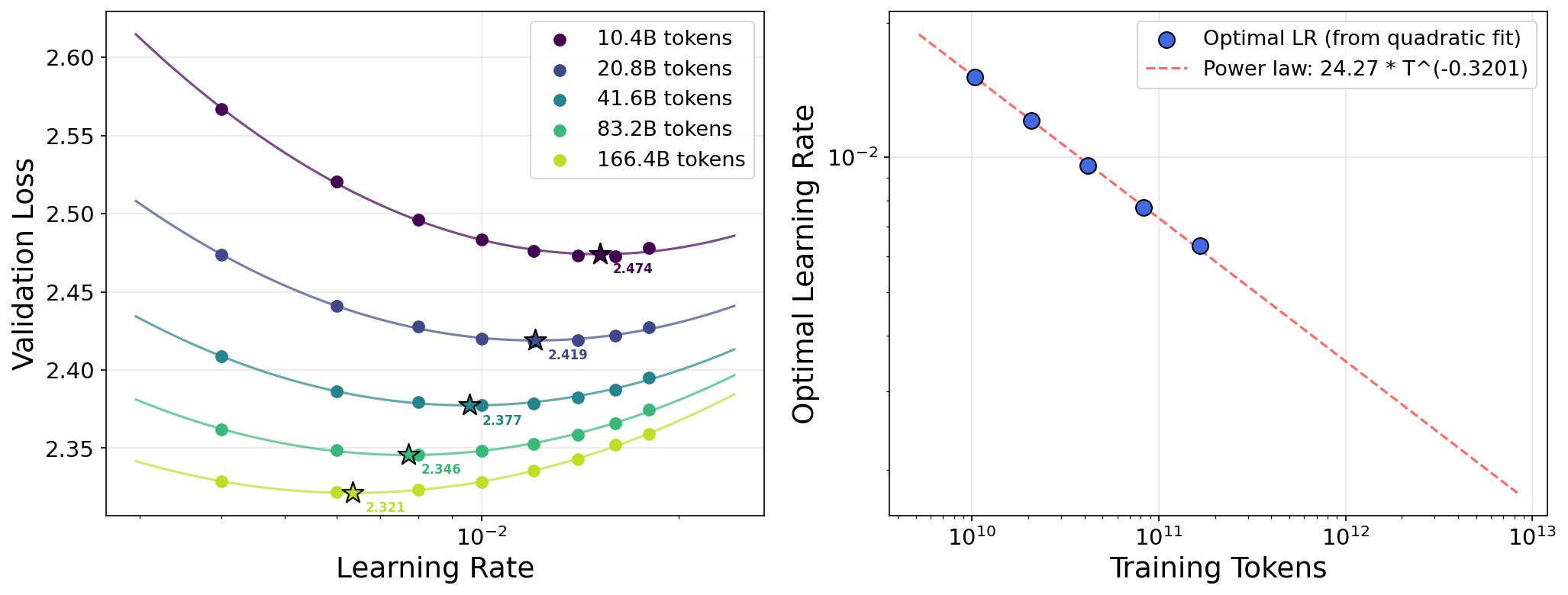}
    \caption{Left: Loss vs.\ LR at different token budgets. Right: Fitted optimal LR vs.\ training tokens on log-log scale, showing a clean power-law relationship with exponent $0.32$. The exact values are reported in \Cref{tab:data-scaling}. }
    \label{fig:data-scaling}
\end{figure}

As shown in \Cref{fig:data-scaling}, the optimal LR follows a clean power law:
\begin{equation}\label{eq:data-scaling}
\eta^* = 24.27 \cdot T^{-0.320}
\end{equation}
where $T$ is the total number of training tokens. 
We also conduct leave-one-out cross-validation, which gives a mean absolute prediction error of only 1.50\% for optimal LR.
The exponent $0.32$ is remarkably consistent with the finding of Bjorck et al.\ \cite{bjorck2025scaling}, who report the same exponent for AdamW on different architectures and datasets. This ``magic exponent'' may be a universal property of gradient-based optimization in neural networks, independent of the specific optimizer. We leave a more rigorous empirical verification and the theoretical analysis of this coincidence as intriguing future work.

\subsection{Hypersphere Parametrization}
We summarize the complete HyperP framework in \Cref{tab:mup-families}, contrasting it with $\mu$P and $\mu$P++ \cite{ren2025decoderhybriddecoder}. HyperP eliminates the weight decay entirely as in \Cref{sec:wd}, inherits native width transfer from the Frobenius-sphere constraint in \Cref{sec:width}, applies the depth scaling derived in \Cref{sec:depth-theory}, and incorporates the data scaling $\eta \propto T^{-0.32}$ established in \Cref{sec:data}.
\begin{table}[htb]
\centering
\small
\caption{Differences between $\mu$P, $\mu$P++ \cite{ren2025decoderhybriddecoder} and HyperP under Muon-based optimizers \cite{jordan2024muon}. \textit{LR mult.} denotes the per-parameter multiplier applied on top of the global learning-rate
($\eta$), \textit{Init. std.} means the standard deviation of the initialization, \textit{Res. mult.}
is the multiplier applied to the output of residual branches and \textit{WD} denotes the weight decay. $w$ is the model width, $d$ means model depth and $T$ is the training tokens. $\mu$P and $\mu$P++ applies Muon for matix-like parameters with adjustments of LR and WD following \cite{chen2025how} and AdamW for unembedding and vector-like parameters, while HyperP adopts MuonH \cite{wen2025hyperball_part1} for matrix-like parameters, AdamH for unembedding weights and AdamW for vector-like parameters.
  } 

\vspace{0.1cm}
\label{tab:mup-families}
\begin{tabular}{lcccccc}
\toprule
\textbf{Parameter} &
\textbf{Scheme} &
\textbf{LR mult.} &
\textbf{Init. std. } &
\textbf{Res. mult.}& \textbf{Weight mult.} & \textbf{WD} \\
\midrule
\multirow{3}{*}{Embedding/Vector} 
 & $\mu$P   & $\propto1$                    &  $\propto1$                           & — & $\propto 1$ & $\propto 1$ \\
 & $\mu$P++ & $\propto 1/\sqrt{d}$                    & $\propto1$                          & — & $\propto 1$ & 0 \\
 & HyperP & $\propto 1/\sqrt{d}$                    & $\propto1$                          & — & $\propto 1$ & 0 \\
[2pt]
\midrule
\multirow{3}{*}{Unembedding} 
 & $\mu$P   & $\propto 1$                    & $\propto1$                           & — &  $\propto 1/w$ & $\propto 1$ \\
 & $\mu$P++ & $\propto 1/\sqrt{d}$                    & $\propto1$                           & — &  $\propto 1/w$ & 0 \\
 & HyperP & $\propto 1/\sqrt{d}$                    & $\propto1$                           & — &  $\propto 1$ & 0 \\
[2pt]
\midrule
\multirow{3}{*}{Hidden Weights } 
 & $\mu$P   & $\propto \sqrt{d_{out}/d_{in}}$                  & $\propto 1/\sqrt{d_{in}}$                              & $1$ & $\propto 1$ & $\propto 1/w$ \\
 & $\mu$P++ & $\propto \sqrt{d_{out}/(d_{in}d)}$                  & $\propto 1/\sqrt{d_{in}}$                                 & $1/\!\sqrt{2d}$ & $\propto 1$ & $\propto 1/w$\\
  & HyperP & $ \propto1/(T^{0.32}\sqrt{d})$                  & $\propto 1/\sqrt{d_{in}}$                                 & $1/\!\sqrt{2d}$ & $\propto 1$ & 0\\

\bottomrule
\end{tabular}
\vspace{-0.3cm}
\end{table}

\subsection{Bounded Logit Magnitudes}\label{sec:zloss}
In standard training, the weight norms can grow unbounded due to the translation-invariance property of Softmax, causing attention, router or LM head logits $z$ to explode. The $Z$-loss penalty $\lambda_z \log^2 Z$ (where $Z = \sum_i \exp(z_i)$) is a common practice \cite{zoph2022st0moe0} introduced to constrain the growth of the log-sum-exponential of the logits.
A key practical benefit of hypersphere optimization is that it naturally bounds the logit magnitudes in both attention and MoE routing, alleviating the need for $Z$-loss regularization. We state the proposition below and provide empirical verification in \Cref{sec:stability}.

\begin{proposition}[Bounded Logits under Hypersphere Constraint]\label{prop:bounded}
For any weight matrix $W$ with $\|W\|_F = C $ and input $x$ with $\|x\|_{\mathrm{rms}} = O(1)$:
\begin{equation}
\|Wx\|_2 \leq  \|W\|_F \|x\|_2 = C \|x\|_2 = C  \|x\|_{\mathrm{rms}} \sqrt{d_{\mathrm{in}}}.
\end{equation}
The per-element logit magnitude is bounded as $|[Wx]_j| \leq C \|x\|_2$, and the RMS of the logit vector satisfies:
\begin{equation}
\|Wx\|_{\mathrm{rms}} \leq C \sqrt{\frac{d_{\mathrm{in}}}{d_{\mathrm{out}}}} \|x\|_{\mathrm{rms}}.
\end{equation}
The proposition similarly applies to the spectral sphere scenario where $\|W\|_2 = C $.
\end{proposition}


\subsection{MoE Granularity Scaling}
\label{sec:moe-theory}

In a Mixture-of-Experts (MoE) \cite{DBLP:conf/iclr/ShazeerMMDLHD17} layer, the output $y$ is formed by combining the Top-$k$ routed experts selected from a larger expert pool. Let $k$ denote the number of active routed experts (the granularity) and let $S$ denote the sparsity ratio, so the layer contains $kS$ routed experts in total. Formally,
\begin{equation}
y = \sum_{i=1}^{k} g_i E_i(x) + E_{\text{shared}}(x),
\qquad \sum_{i=1}^{k} g_i = 1,
\end{equation}
where $g_i$ are the routing weights over the selected $k$ experts, following the design of \cite{DBLP:conf/iclr/ShazeerMMDLHD17, openai2025gpt0oss0120b}, and $E_{\text{shared}}$ is an optional shared expert \cite{deepseek-ai2024deepseekv2}.

\begin{proposition}[Classical gating is $k$-dependent]
\label{prop:moe-classical-gating}
Let
$
y_{\mathrm{route}}=\sum_{i=1}^{k} g_i E_i(x).
$
Assume the active expert outputs satisfy
$
\|E_i(x)\|_{\mathrm{rms}} = r  \, \,\text{for all } i,
$
and are approximately pairwise uncorrelated:
$
\langle E_i(x), E_j(x)\rangle \approx 0 \,\, \text{for } i\neq j.
$
Then
\begin{equation}
\|y_{\mathrm{route}}\|_{\mathrm{rms}}
\approx
r \sqrt{\sum_{i=1}^{k} g_i^2}.
\end{equation}
In particular, if the routing weights are near-uniform on the selected experts, i.e. $g_i \approx 1/k$, then
\begin{equation}
\|y_{\mathrm{route}}\|_{\mathrm{rms}}
\approx
\frac{r}{\sqrt{k}}.
\end{equation}
By contrast, $\|y_{\mathrm{route}}\|_{\mathrm{rms}} \approx r$ is recovered only in the degenerate case where routing is nearly one-hot, i.e. one $g_i\approx 1$ and the others are close to zero.
\end{proposition}
This shows that classical softmax gating preserves RMS only in the worst-case collapsed-routing regime. In the more typical case where multiple selected experts contribute non-trivially, the routed signal shrinks with $k$. In our setting, hypersphere optimization makes the equal-RMS assumption natural by explicitly controlling the output scale of each expert with weight normalization. Moreover, Muon optimizer can indirectly reduce co-adaptation across experts by reducing anisotropy in each expert's matrix updates, which makes the pairwise-uncorrelated approximation more realistic. This motivates us to analyze the routed branch under the equal-RMS, weak-correlation regime rather than under the worst-case scenario. We propose to replace $g_i$ with $\sqrt{g_i}$, and denote our approach as \textbf{SqrtGate} (Square-root Gate).

\begin{proposition}[SqrtGate is approximately $k$-invariant]
\label{prop:moe-sqrt-gating}
Define the routed branch by
$
y_{\mathrm{route}}'
=
\sum_{i=1}^{k} \sqrt{g_i}\, E_i(x),
\, \sum_{i=1}^{k} g_i = 1.
$
Under the same assumptions as in Proposition~\ref{prop:moe-classical-gating},
\begin{equation}
\|y_{\mathrm{route}}'\|_{\mathrm{rms}}
\approx
r \sqrt{\sum_{i=1}^{k} (\sqrt{g_i})^2}
=
r.
\end{equation}
Hence the routed-branch RMS is approximately invariant to the granularity $k$.
\end{proposition}

We can see that classical gating is RMS-preserving only when Top-$k$ routing effectively collapses to Top-1, whereas SqrtGate is RMS-preserving for any gating distributions in the equal-RMS, weak-correlation regime induced by hypersphere optimization.  When the shared expert is presented, we also multiply $1/\sqrt{2}$ to the final output $y$ to preserve the overall output RMS after summation. 

\section{Experiments \& Results}\label{sec:experiments}

\paragraph{Architecture.} Throughout this work, we use the \texttt{Transformer-Next} architecture family, inspired by the attention module design in Qwen3-Next \cite{qwen2025qwen3next80ba3b}: dense Transformers with GQA (4 KV heads) \cite{ainslie2023gqa}, head dimension 128, aspect ratio $\alpha =128$ (\emph{i.e.} model width $w = 128d$), QK-Norm \cite{pmlr-v202-dehghani23a}, and headwise gated attention \cite{qiu2025gated}. The number of attention heads is set to $n_{\mathrm{head}} = 2d$, where $d$ is the model depth, so that $n_{\mathrm{head}} $ is always a multiple of 8 during scaling. We use SwiGLU \cite{shazeer2020glu} with $4w$ intermediate size  \cite{ren2025decoderhybriddecoder,openai2025gpt0oss0120b} for MLP, and apply Pre-Norm \cite{prenorm} for residual connections. The MoE module follows the same design as in \Cref{sec:moe-theory} with SqrtGate and a shared expert, where the Softmax operator is after Top-k selection \cite{openai2025gpt0oss0120b}, and we denote this architecture as \texttt{Transformer-Next-MoE}. We sweep depths $d \in \{8, 12, 16, 20, 24\}$, corresponding to 208M--3.8B parameters for the dense model and 913M--22.9B total parameters for MoE models with a sparsity of 8.
To match the active parameters to the dense model, we (1) choose Top-($k-1$) experts from an expert pool of $kS-1$ experts and have 1 shared expert always activated, and (2) shrink the intermediate dimension of the experts as we scale up the granularity.

\paragraph{Training setup.} By default, all models are trained on the SlimPajama dataset \cite{slimpajama} with a context length of 4K and a batch size of 2M tokens. The learning rate schedule uses a linear decay to 10\% of peak without warm-up, following \cite{modded_nanogpt_2024}. A momentum of 0.95 is adopted for both Muon and MuonH. For FLOPs scaling, the number of training tokens is scaled proportionally to number of parameters according to Chinchilla Law \cite{chinchilla} with a measure of Tokens Per Parameter (TPP) $= T / N$, where $T$ is the total training tokens and $N$ is the parameter count.  The PyTorch \cite{paszke2019pytorch} default initialization from Kaiming uniform distribution \cite{resnet} is adopted. The independent weight decay \cite{wortsman2024smallscale} is applied for the Muon optimizer.

\paragraph{Scaling comparison and compute efficiency leverage.}
We follow the Chinchilla law \cite{chinchilla} for fine-grained FLOPs computation, which accounts for embedding and language-model head FLOPs, as well as an accurate self-attention FLOPs calculation.
To compare scaling behaviors, we follow \cite{kaplan2020scaling, team2025every} to fit each method's final validation loss as a power law in training FLOPs, $C$, then define the compute efficiency leverage $\rho = C_{\mathrm{base}} / C^{*}$, where $C^{*}$ is the method's actual FLOPs and $C_{\mathrm{base}}$ is the FLOPs the baseline would need to reach the same loss $L^{*}$ of the method according to its fitted scaling law; $\rho > 1$ indicates better compute efficiency than the baseline.


\subsection{Empirical Optimality of MuonH}\label{sec:optimality}

A natural concern of hypersphere optimization is whether removing weight decay trades off performance. We compare MuonH against standard Muon at $d{=}8$ dense model with 10.4B tokens. For Muon, we jointly sweep learning rate $\eta \in \{4{\times}10^{-3}, 8{\times}10^{-3}, 10^{-2}, 2{\times}10^{-2}, 4{\times}10^{-2}\}$, and weight decay $\lambda \in \{4{\times}10^{-4}, 8{\times}10^{-4}, 10^{-3}, 2{\times}10^{-3}, 4{\times}10^{-3}\}$; for MuonH, weight decay is set to 0.

\begin{figure}[H]
    \centering
    \includegraphics[width=\linewidth]{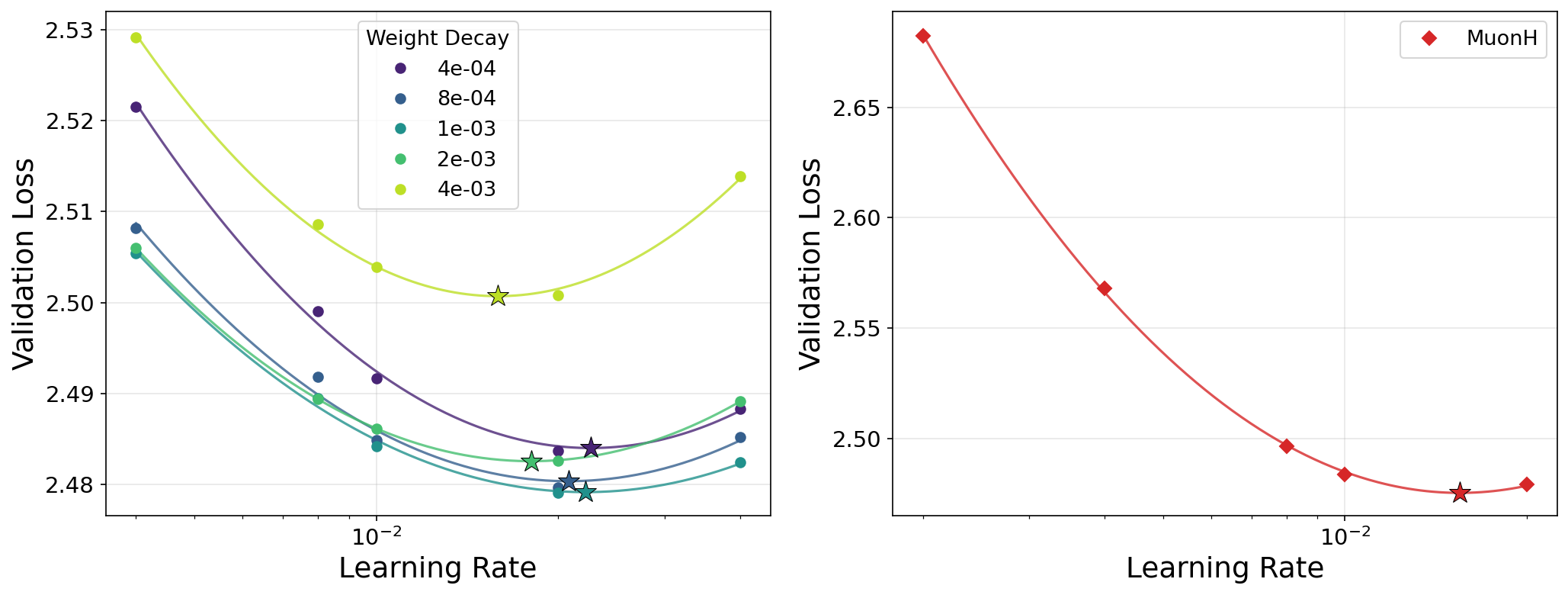}
    \caption{Validation loss vs.\ learning rate for Muon (sweeping weight decay $\lambda$) and MuonH ($\lambda{=}0$). MuonH achieves comparable optimality with a simpler hyperparameter space.}
    \label{fig:muonh-vs-muon}
\end{figure}

As shown in \Cref{fig:muonh-vs-muon} and \Cref{tab:muonh-vs-muon}, MuonH achieves a slightly better validation loss while entirely removing weight decay as a hyperparameter. The optimal LR for MuonH is ${\sim}1.4\times$ smaller than for Muon. Muon's performance is sensitive to weight decay: the best $\lambda = 10^{-3}$ gives a loss of 2.479, while $\lambda = 4{\times}10^{-3}$ gives 2.500 ($+0.021$ nats). These empirical results mean that MuonH does not trade-off quality for a simpler hyperparameter space and support our theory on the weight decay elimination effect of Frobenius-sphere optimization in \Cref{thm:fnorm_tangent_update}.
\vspace{-0.3cm}
\begin{table}[H]
\centering
\caption{Fitted optimal learning rate $\eta*$ and validation loss between MuonH and Muon. MuonH matches Muon while eliminating weight decay.}
\label{tab:muonh-vs-muon}
\vspace{0.1cm}
\begin{tabular}{lccc}
\toprule
\textbf{Method} & \textbf{Fitted $\eta^*$} & \textbf{Best Val Loss} & \textbf{Weight Decay} \\
\midrule
Muon (best $\lambda{=}10^{-3}$) & 0.0222 & 2.479 & $10^{-3}$ \\
MuonH ($\lambda{=}0$) & 0.0155 & 2.475 & 0 \\
\bottomrule
\end{tabular}
\end{table}

\subsection{Parameter Scaling}\label{sec:param-scaling}

We empirically verify the depth scaling predictions of \Cref{sec:depth-theory} by co-scaling width and depth at a fixed aspect ratio ($w = \alpha d$, $\alpha = 128$), so that the model is well-shaped across scales. We run all depth experiments at 10.4B tokens and sweep learning rates on a fine grid $\{0.002, 0.004, \ldots, 0.020\}$. Figure~\ref{fig:depth-mup} compares MuonH with and without Depth-$\mu$P at 50 TPP across $d \in \{8, 12, 16, 20, 24\}$ on the same LR grid, while full LR-loss sweeps are reported in \Cref{tab:depth-scaling,tab:depth-scaling-mup} and a summary of optimal values is provided in \Cref{tab:depth-mup}.

\begin{figure}[H]
    \centering
    \includegraphics[width=\linewidth]{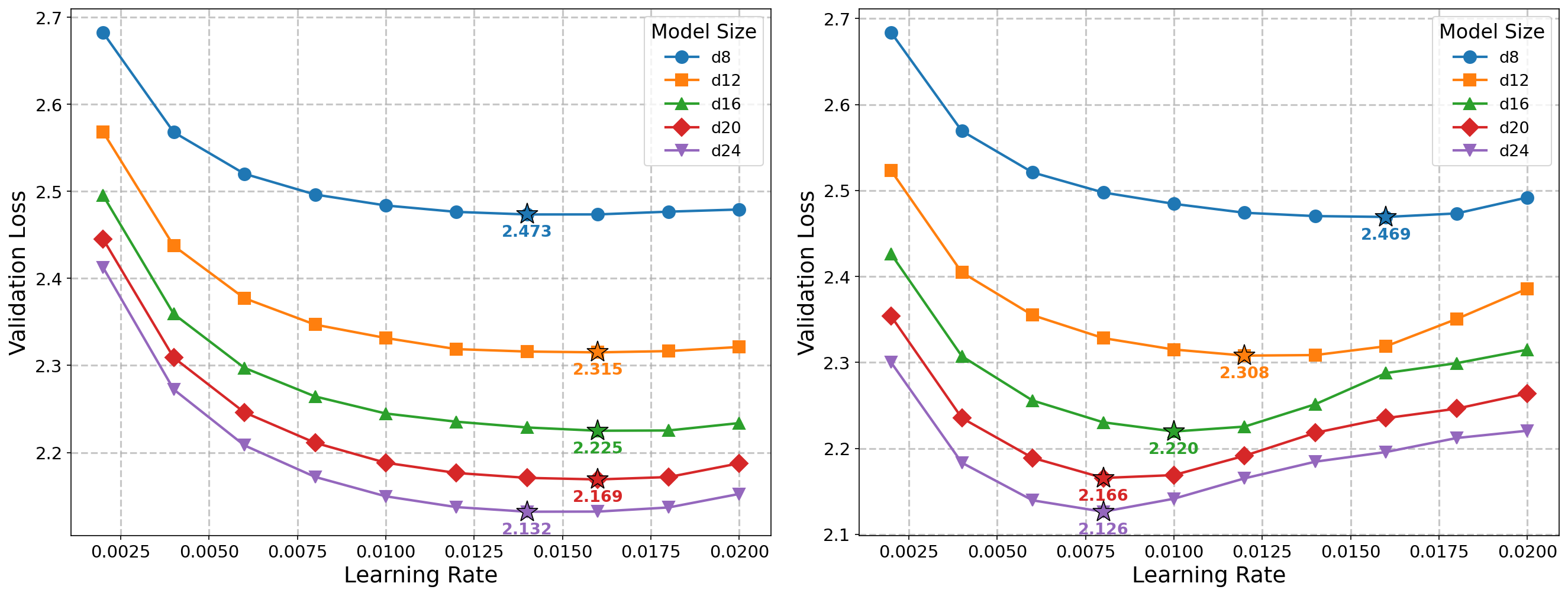}
    \caption{Loss vs.\ LR curves across model sizes with Depth-$\mu$P (left) and without Depth-$\mu$P (right). Depth-$\mu$P keeps the optimal LR stable at $\eta^* \approx 0.014$ -- $0.016$ across all depths, while the optimum drifts from $\eta^* = 0.016$ at $d{=}8$ to $\eta^* = 0.008$ at $d{=}24$ without Depth-$\mu$P.}
    \label{fig:depth-mup}
\end{figure}

Without Depth-$\mu$P, the optimal learning rate decreases from $\eta^* = 0.016$ at $d{=}8$ to $\eta^* = 0.008$ at $d{=}24$, consistent with the depth-dependent LR trend predicted in \Cref{sec:depth-theory}; the loss landscape also sharpens with depth, as increasing LR from the optimum to $\eta=0.020$ incurs a $+0.023$ nats penalty at $d{=}8$ (2.492 vs.\ 2.469) but a $+0.098$ nats penalty at $d{=}20$ (2.264 vs.\ 2.166). In contrast, with Depth-$\mu$P the optimal LR remains nearly constant at $\eta^* \approx 0.014$ -- $0.016$ from $d{=}8$ to $d{=}24$. Crucially, both configurations achieve comparable best losses at each depth (\Cref{tab:depth-mup}), confirming that Depth-$\mu$P preserves model quality while enabling hyperparameter transfer. These results empirically validate our theory in \Cref{sec:depth-theory} and refute the claim that MuonH is inherently depth-transferable \cite{wen2025hyperball_part1}.

\subsection{Critical Batch Size}\label{sec:bsz}

We fix $d{=}8$ with 10.4B tokens and sweep LR across batch sizes $B \in \{256\text{K}, 512\text{K}, 1\text{M}, 2\text{M}\}$ tokens to identify the critical batch size, the threshold above which increasing batch size significantly degrades the achievable loss \cite{mccandlish2018empirical}.

\begin{figure}[H]
    \centering
    \includegraphics[width=\linewidth]{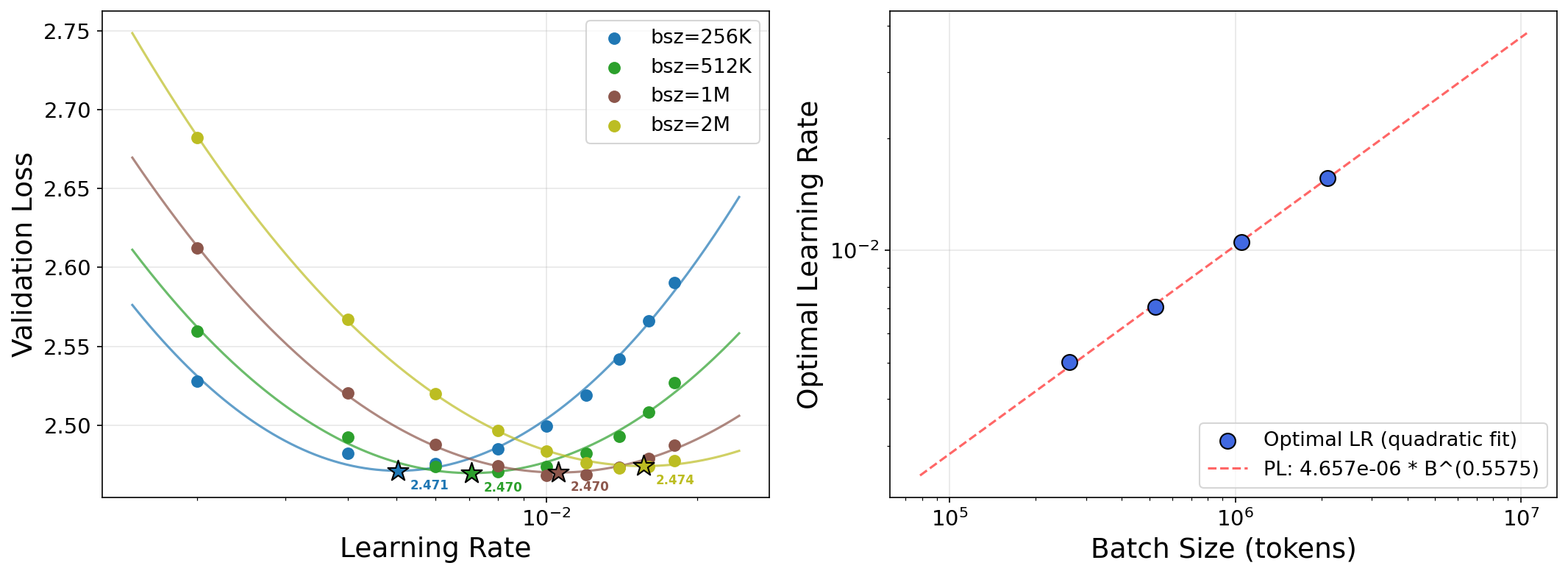}
    \caption{Left: Loss vs.\ LR at different batch sizes. Right: Optimal LR vs.\ batch size on log-log scale. The exact values are reported in \Cref{tab:bsz-scaling}.}
    \label{fig:bsz-scaling}
\end{figure}

The optimal LR scales as $\eta^* = 4.66 \times 10^{-6} \cdot B^{0.558}$, with exponent $\approx 0.56$ sitting between the linear scaling rule (exponent 1.0) and the square-root rule (exponent 0.5, predicted by SDE analysis \cite{malladi2022on}). The minimum achievable loss is remarkably stable across batch sizes (within a 0.001 difference across batch sizes from 256K to 1M and within 0.004 for 2M), indicating that all tested batch sizes are below the critical batch size for this configuration. Since the optimal loss is mostly invariant under the tested batch size, we fix the batch size to 2M tokens for all subsequent experiments so that the batch size does not confound the scaling behavior. We leave the study of the relationship between critical batch size and training tokens to future work, as it requires a straightforward but costly scaling of the same suite of experiments explored in this section across multiple token budgets.


\subsection{MoE Scaling}\label{sec:moe}

We extend our empirical verification of HyperP to our \texttt{Transformer-Next-MoE} architecture, investigating when scaling sparsity and granularity plateaus under optimal learning rates.

\paragraph{Auxiliary Balance Loss.}\label{sec:auxloss}
We apply the Switch-Transformer load balancing loss \cite{fedus2022switch}, computed over the global batch across all data-parallel ranks:
\begin{equation}\label{eq:aux-loss}
    \mathcal{L}_{\text{aux}} = \gamma \cdot N \cdot \sum_{i=1}^{N} f_i \cdot P_i,
\end{equation}
where $N$ is the number of experts, $f_i = c_i / \sum_{j=1}^{N} c_j$ is the fraction of tokens dispatched to expert~$i$ (with $c_i$ the hard count), and $P_i = \sum_{t=1}^{T} p_{t,i}\, /\, \sum_{j=1}^{N} \sum_{t=1}^{T} p_{t,j}$ is the normalized total router probability for expert~$i$, with $p_{t,i}$ the post-softmax routing weight for token~$t$. Both $f_i$ and $P_i$ are aggregated across all ranks via all-reduce before computing the loss. Under 10.4B training token budgets with $d=8$ model size, we sweep the auxiliary balance loss weight $\gamma \in \{10^{-3}, 10^{-2}, 10^{-1}\}$ for the $S=8, k=4$ MoE configuration,as shown in \Cref{fig:auxloss}. Note that in this setting, we train \emph{without} SqrtGate or shared experts to remove architectural confounders, allowing us to better isolate the effect of hypersphere optimization on load balancing.
\begin{figure}[H]
    \centering
    \includegraphics[width=0.5\linewidth]{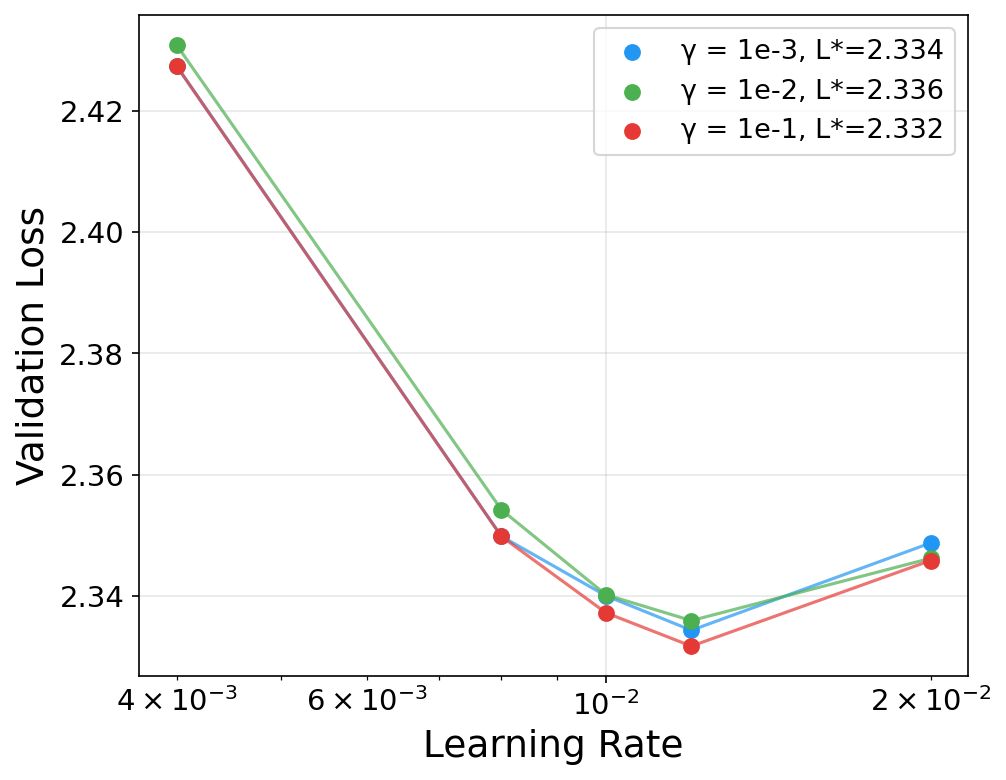}
    \caption{Loss vs.\ LR curves for three auxiliary loss weights. The curves nearly overlap, indicating robustness on $\gamma$ under hypersphere optimization. The exact values across all LR and $\gamma$ combinations are reported in \Cref{tab:auxloss}. }
    \label{fig:auxloss}
\end{figure}
To quantify load imbalance, we use the MaxVio (maximal violation) metric \cite{lin2024auxloss}, which measures how much the most loaded expert exceeds the balanced baseline:
\begin{equation}\label{eq:maxvio}
    \mathrm{MaxVio} = \frac{\max_{i} \, c_i - \bar{c}}{\bar{c}}, \qquad \bar{c} = \frac{1}{N}\sum_{i=1}^{N} c_i,
\end{equation}
where $c_i$ is the number of tokens dispatched to expert~$i$ within a batch and $\bar{c}$ is the expected load under perfect balance. A value of zero indicates perfectly uniform routing. We compute MaxVio per layer and report the mean across layers (Mean MaxVio).
Surprisingly, as shown in \Cref{tab:load-balance}, the largest weight $\gamma = 10^{-1}$ achieves the \emph{best} loss (2.332) with the lowest Mean MaxVio. This contrasts with prior work suggesting that the auxiliary loss harms language modeling quality and thereby motivates auxiliary-loss-free load balancing \cite{lin2024auxloss}. Under hypersphere optimization, the bounded logits (\Cref{prop:bounded}) likely prevent the auxiliary loss from interfering with the language modeling objective, and we leave the theoretical study to future work. Motivated by these results, we adopt $\gamma = 0.1$ for all experiments in the remainder of the paper.

\begin{table}[H]
\centering
\caption{Effect of $\gamma$ on global load balancing. Mean max violation measures worst-case load imbalance across experts at $\eta^* = 0.012$ from the global batch. } 
\label{tab:load-balance}
\vspace{0.1cm}
\begin{tabular}{ccc}
\toprule
$\gamma$ & \textbf{Best Val Loss} & \textbf{Mean MaxVio} \\
\midrule
$10^{-3}$ & 2.334 & 0.848 \\
$10^{-2}$ & 2.336 & 0.132 \\
$10^{-1}$ & 2.332 & 0.086 \\
\bottomrule
\end{tabular}
\end{table}


\paragraph{Sparsity Scaling.}\label{sec:sparsity}
We sweep sparsity $S \in \{1, 2, 4, 8, 16, 32\}$ for \texttt{Transformer-Next-MoE} with granularity $k=4$. The model has a constant 208M active parameters, with a range of 208M--3.33B total parameters depending on the sparsity. As shown in \Cref{fig:sparsity}, the optimal LR varies only mildly (0.012--0.016) across a $32\times$ sparsity range, indicating strong LR transferability over MoE sparsity. Increasing sparsity consistently improves validation loss: moving from $S{=}1$ to $S{=}32$ reduces the optimal loss by 0.224. This means that it would require approximately $5.2\times$ more dense training FLOP to achieve the same loss reduction under the MuonH+HyperP scaling curve in \Cref{fig:flops-scaling}.

\begin{figure}[H]
    \centering
    \includegraphics[width=0.95\linewidth]{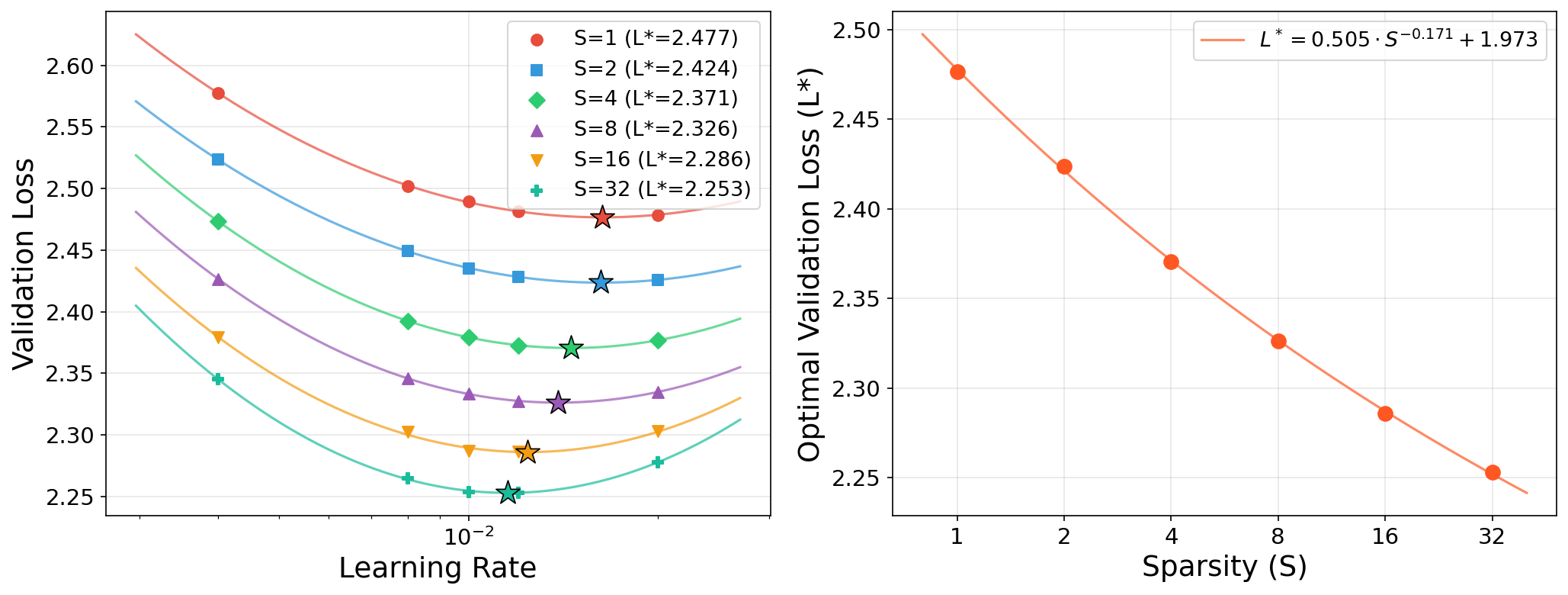}
    \caption{Left: Loss vs.\ LR across sparsity levels. Right: Optimal loss follows a power law in the MoE sparsity. The exact values are reported in \Cref{tab:sparsity}. }
    \label{fig:sparsity}
\end{figure}

\paragraph{Granularity Scaling.}\label{sec:topk}
We sweep $k \in \{2, 4, 8, 16, 32, 64\}$ with and without SqrtGate to verify that our granularity scaling theory in \Cref{sec:moe-theory} holds in practice. As shown in \Cref{fig:topk}, the optimal LR varies only between 0.012--0.014 across a $32\times$ range of $k$, enabling direct LR transfer across MoE granularity configurations. SqrtGate consistently improves val loss at every $k$, with the largest gain at $k = 2$ ($-0.018$ nats) and $k=32$ ($-0.009$ nats).  With SqrtGate, performance continues to improve up to $k=32$, which achieves the best loss of 2.310, whereas the baseline saturates at $k=16$. These results show that SqrtGate improves both the model quality and the granularity scalability compared to the baseline.

\begin{figure}[H]
    \centering
    \includegraphics[width=0.95\linewidth]{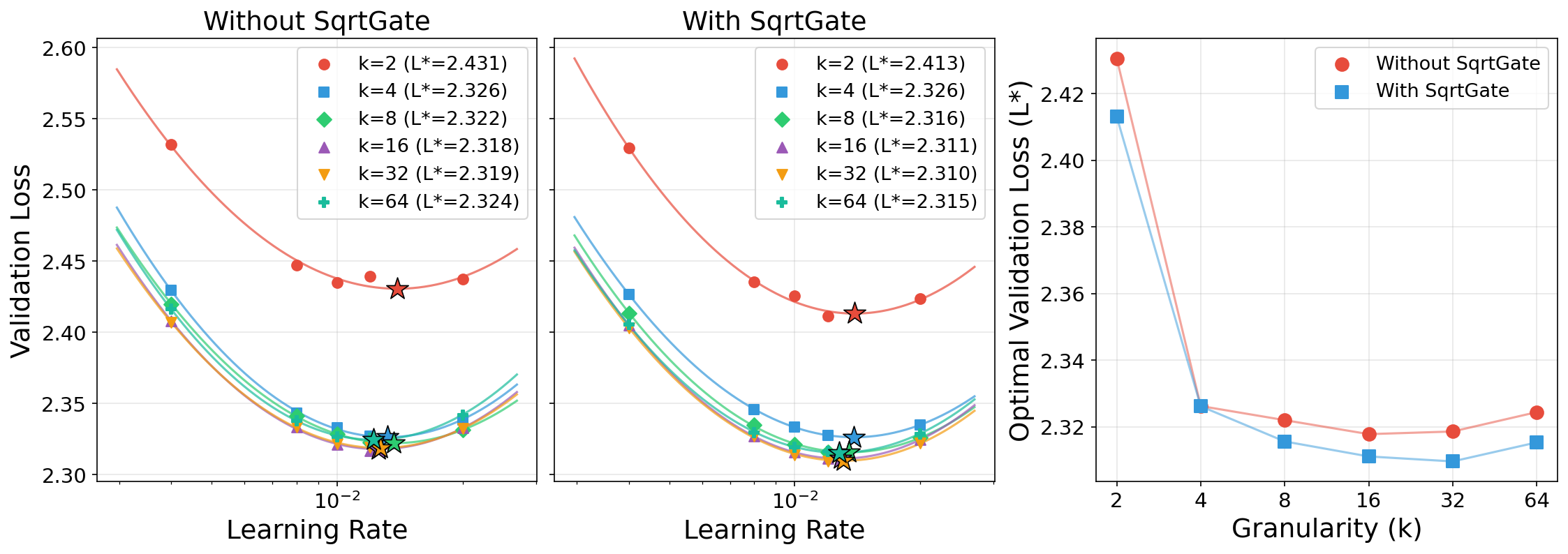}
    \caption{Loss vs.\ LR across top-$k$ values with and without SqrtGate.  The exact optimal learning rates and losses are provided in \Cref{tab:topk}.}
    \label{fig:topk}
\end{figure}

\subsection{Training FLOPs Scaling}\label{sec:flops}
\paragraph{Empirical optimality of HyperP across FLOPs.}
Before comparing scaling behaviors between different hyperparameter transfer laws, we first verify that HyperP preserves the empirical optimality of a single small-scale LR choice as training FLOPs increase. As shown in \Cref{fig:hyperp-transfer}, the loss-vs-LR curves under HyperP remain well aligned from $d{=}8$ through $d{=}20$, and the same base optimum $\eta_0 = 0.02$ is preserved across scales. This confirms the central premise of HyperP: one LR sweep at small scale is sufficient to determine the learning rates used along the full scaling trajectory. In contrast, without HyperP the optimal learning rate drifts with depth, so directly reusing the small-scale learning rate becomes increasingly miscalibrated and leads to substantially worse performance.

\begin{figure}[H]
    \centering
    \includegraphics[width=\linewidth]{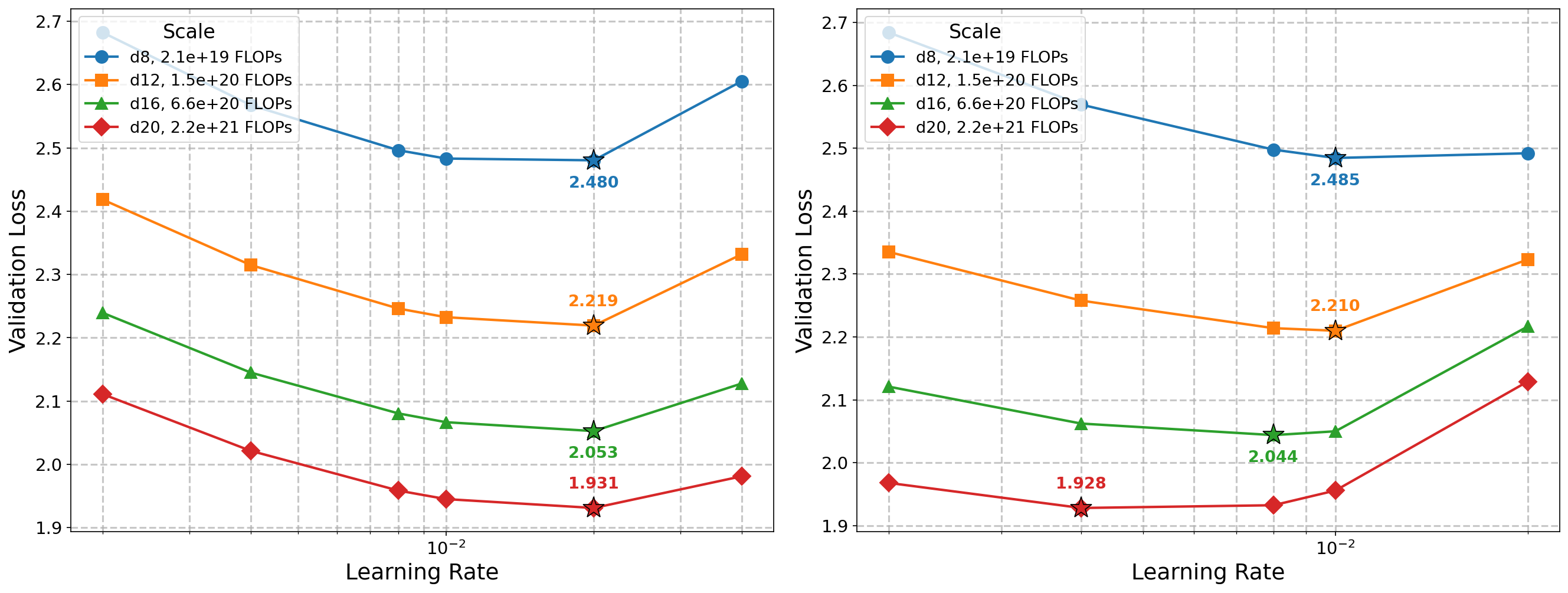}
    \caption{Loss vs.\ LR across depths with HyperP (left) and without HyperP (right). HyperP keeps the curves aligned and preserves a common base optimum at $\eta_0 = 0.02$ from $d{=}8$ through $d{=}20$.}
    \label{fig:hyperp-transfer}
\end{figure}
\begin{figure}[H]
    \centering
    \includegraphics[width=\linewidth]{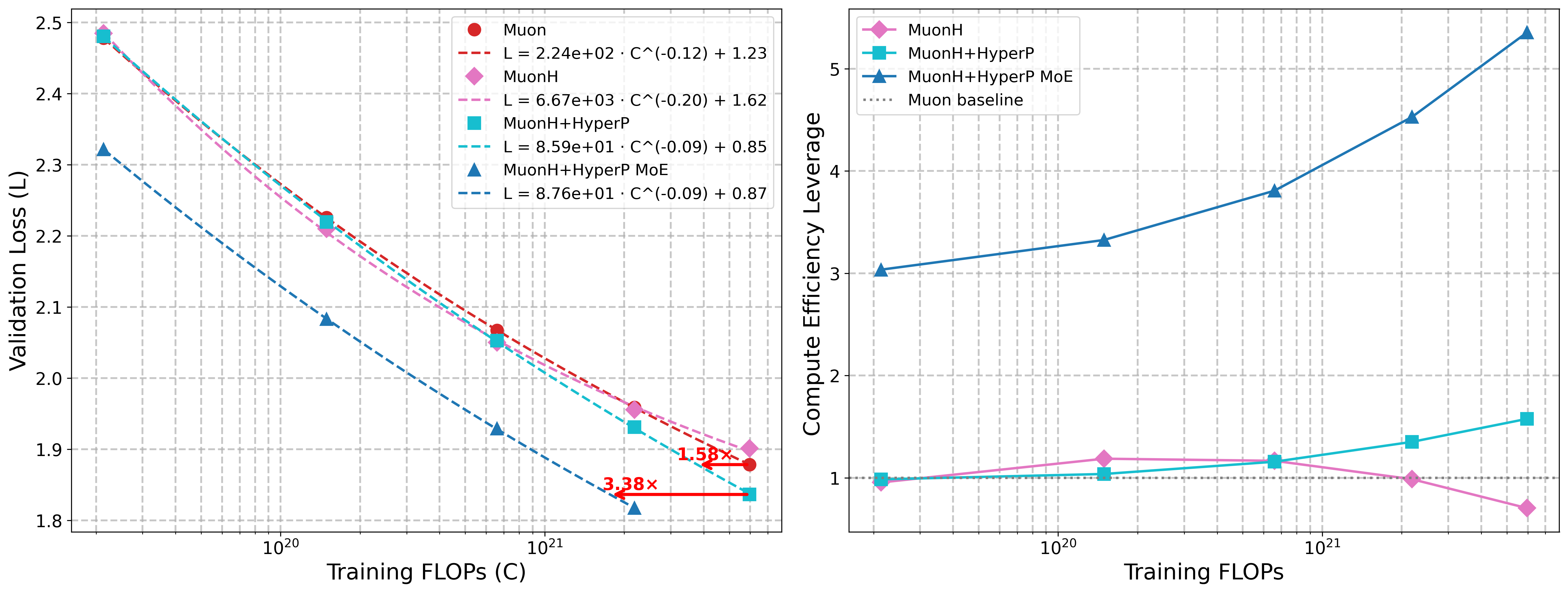}
    \caption{Left: Loss vs.\ FLOPs with power-law fits for all four methods. Right: Compute Efficiency Leverage (CEL) relative to the Muon baseline. MuonH+HyperP achieves 1.58$\times$ CEL than the MuonH baseline, while the MuonH+HyperP MoE models achieve 3.38$\times$ CEL over the dense model baselines. The exact values are reported in \Cref{tab:flops-scaling}.}
    \label{fig:flops-scaling}
\end{figure}

\paragraph{Compute scaling comparisons.} In \Cref{fig:flops-scaling}, we compare the end-to-end compute scaling behaviors of various hyperparameter transfer laws. Each method is tuned once at the smallest model size $d{=}8$ using a coarse-grained LR sweep $\eta \in \{2{\times}10^{-3}, 4{\times}10^{-3}, 8{\times}10^{-3}, 1{\times}10^{-2}, 2{\times}10^{-2},4{\times}10^{-2}\}$ and then scaled with the observed optimal learning rate across model sizes $d \in \{8,12,16,20,24\}$ with 50 TPP. Specifically, we compare four configurations:
\begin{itemize}[leftmargin=1.5em,itemsep=1pt]
    \item \textbf{Muon}: $\mu$P++ \cite{ren2025decoderhybriddecoder} with $\propto 1/w$ weight decay scaling \cite{chen2025how}, using optimal base LR $\eta^*=0.02$  with base weight decay $\lambda^*=10^{-3}$.
    \item \textbf{MuonH}: vanilla MuonH with $\propto 1/\sqrt{d_{in}}$ initialization \cite{wen2025hyperball_part1}, using optimal base LR $\eta^*=0.01$.
    \item \textbf{MuonH+HyperP}: MuonH with HyperP, using optimal base LR $\eta^*=0.02$.
    \item \textbf{MuonH+HyperP MoE}: HyperP applied to the \texttt{Transformer-Next-MoE} model with $S{=}8$ and $k{=}8$, using the optimal base LR $\eta^*=0.01$.
\end{itemize}

Given the empirical results, we fit each Loss--FLOPs trajectory with \(L = A \cdot C^{-b} + C_0\) (\Cref{fig:flops-scaling}, left). Among dense models, MuonH+HyperP exhibits the strongest scaling trends, achieving the lowest irreducible floor (\(C_0 = 0.85\)), compared with \(1.23\) for Muon and \(1.62\) for MuonH without HyperP. At the largest budget (\(5.96\times10^{21}\) FLOPs), MuonH+HyperP attains \(1.58\times\) Compute Efficiency Leverage (CEL) over the Muon baseline. The MuonH+HyperP MoE model achieves a similarly low floor (\(C_0 = 0.87\)) while outperforming all dense models across the full FLOPs range, reaching up to \(3.38\times\) CEL over the dense baselines at the largest budget. The CEL of MuonH+HyperP increases monotonically with scale (\Cref{fig:flops-scaling}, right), rising from near parity at \(d{=}8\) to \(1.58\times\) leverage at \(d{=}24\). In contrast, MuonH without HyperP is briefly competitive at intermediate scales but ultimately declines to \(0.70\times\), showing that even a modest learning-rate transfer mismatch compounds into a substantial compute efficiency penalty at large scale.

\section{Analysis}\label{sec:analysis}
\subsection{Transferable Stability}\label{sec:stability}

\begin{figure}[H]
    \centering
    \includegraphics[width=\linewidth]{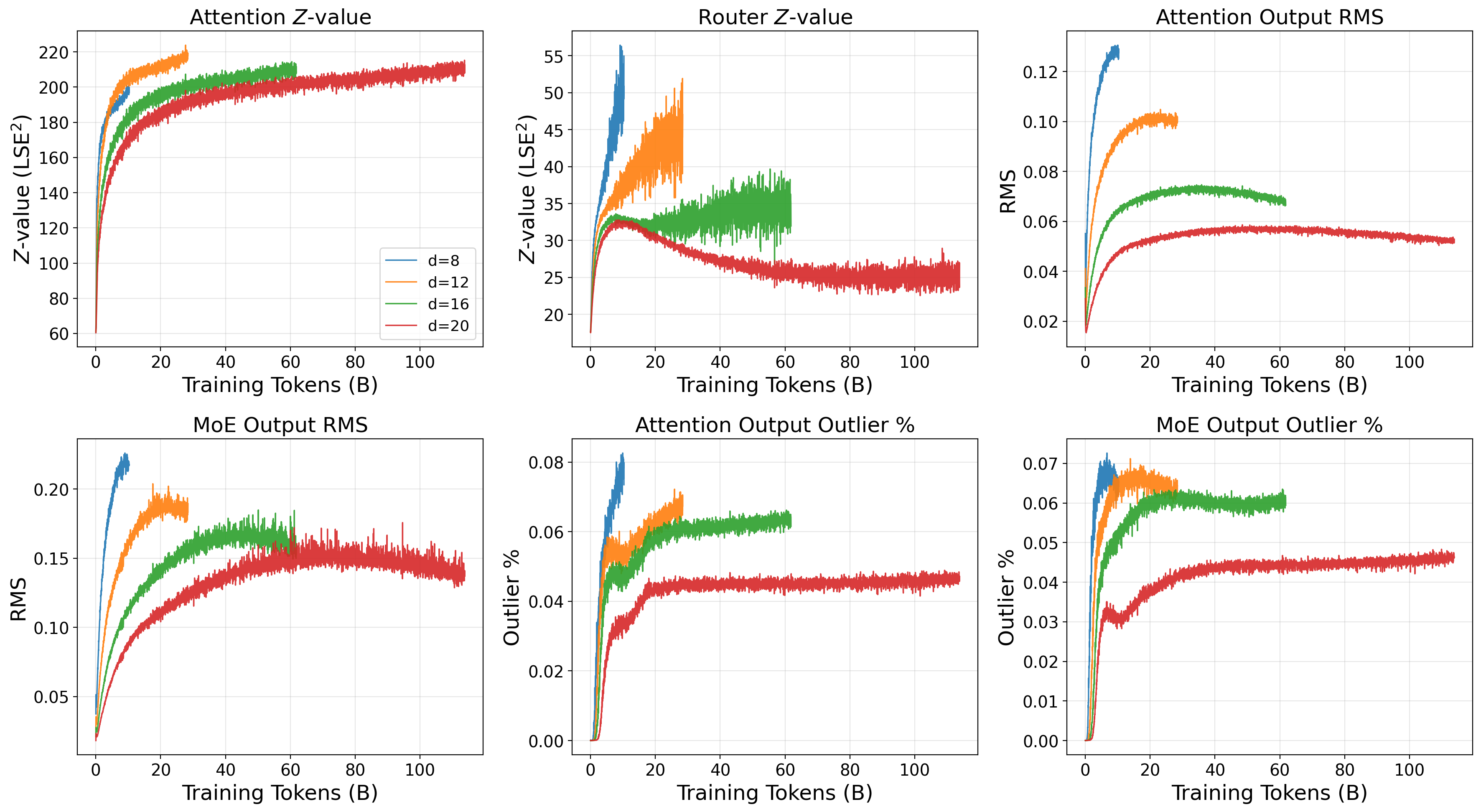}
    \caption{Stability metrics across training for MoE models at depths $d \in \{8, 12, 16, 20\}$ under the same transferred LR. All metrics are bounded and non-increasing with scales. 
    }
    \label{fig:stability}
\end{figure}

A practical concern with hyperparameter transfer is whether training stability degrades at larger scales when the hyperparameters are configured using a small proxy. We track six metrics of the training configuration \emph{MuonH+HyperP MoE} conducted in \Cref{sec:flops} for the \texttt{Transformer-Next-MoE} model across depths $d \in \{8, 12, 16, 20\}$:
\begin{itemize}[leftmargin=1.5em,itemsep=2pt]
    \item \textbf{$Z$-values} ($\mathrm{LSE}^2$): For both attention and MoE routing, we compute $Z = \frac{1}{BT}\sum \mathrm{LSE}(\mathbf{z})^2$, where $\mathrm{LSE}(\mathbf{z}) = \log\sum_i \exp(z_i)$ is the log-sum-exp of the pre-softmax logits. This is the quantity penalized by $Z$-loss \cite{zoph2022st0moe0}; large $Z$ indicates logit explosion.
    \item \textbf{Output RMS}: The root-mean-square magnitude of the attention and MoE residual-branch outputs, averaged across layers. Growing output norms signal representational instability.
    \item \textbf{Outlier \%}: The fraction of hidden-state elements of the attention and MoE residual-branch outputs exceeding $5\sigma$ from the per-token mean, averaged across layers. 
    This detects the emergence of activation outliers that degrade quantization.
\end{itemize}

\Cref{fig:stability} shows that all six metrics are well-behaved as scale increases. The attention $Z$-values plateau at comparable magnitudes across depths ($\approx$200--220). The router $Z$-values are even better behaved: their peaks decrease monotonically with depth (from 56 at $d{=}8$ to 33 at $d{=}20$) and continue to decline during training at the largest scale. 
The output RMS norms decrease with depth for both attention and MoE residual branches.
The outlier percentages similarly decrease with depth, indicating that larger models under HyperP do not develop more activation outliers commonly observed in standard training \cite{dettmers2022llm0int8000}. These results show that HyperP provides not only optimal LR transfer but also \emph{stability transfer}: the same hyperparameters that work well at small scales produce equally or more stable training dynamics at large scales.

\subsection{Sensitivity of Optimal Learning Rate Estimation}\label{sec:sensitivity}

Since we heavily use quadratic fitting to find optimal Learning Rates (LRs), we want to understand how many LR sweep points are needed for reliable estimates. We take the experimental data from our data scaling experiments in \Cref{sec:data} as a case study. We enumerate all $\binom{8}{k}$ combinations of $k$ points from 8 available LRs, fit the parabola in $\log(\eta)$ space, and measure the mean relative error compared to the full 8-point fit.

\begin{figure}[H]
    \centering
    \includegraphics[width=0.95\linewidth]{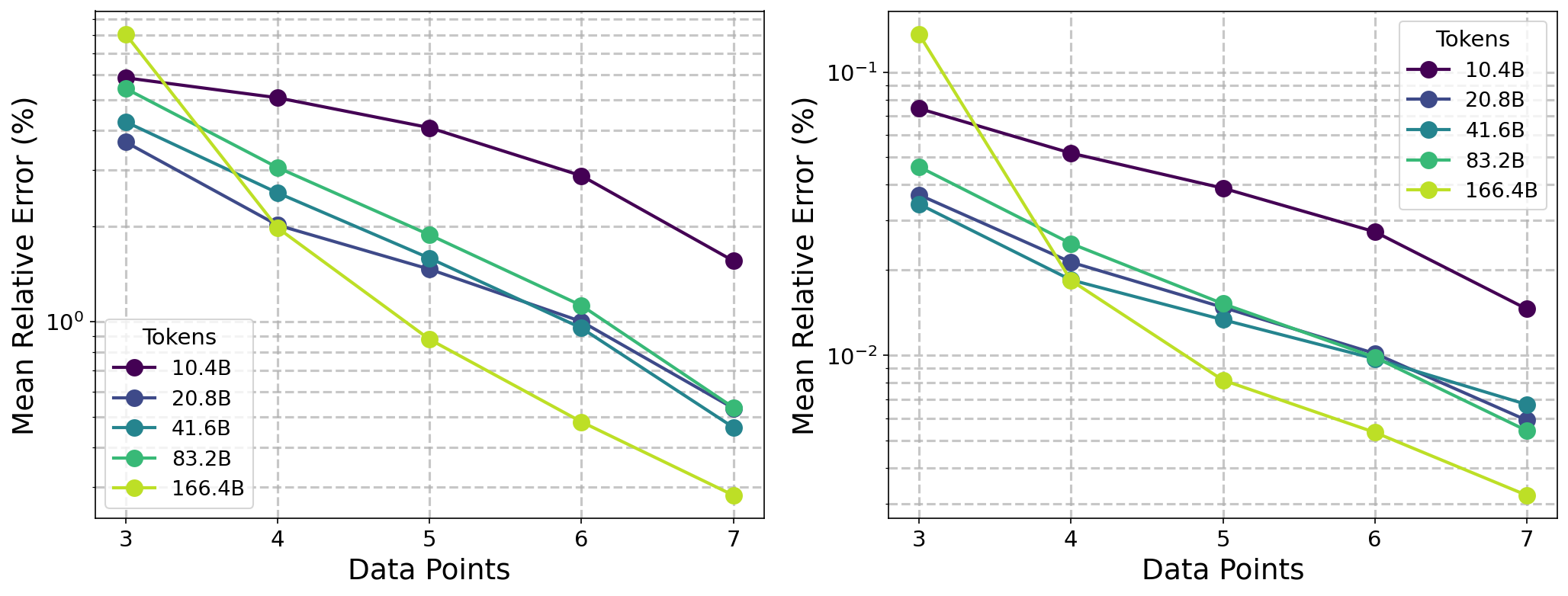}
    \caption{Relative error in optimal LR (Left) and optimal loss (Right) estimates vs.\ number of sweep points. The exact values are reported in \Cref{tab:sensitivity}.}
    \label{fig:sensitivity}
\end{figure}

\textbf{Optimal loss is far more stable than optimal LR.} The loss estimate is consistently $50$--$140\times$ less sensitive than the LR estimate (\Cref{tab:sensitivity}). With only $n{=}3$ points, the loss error is 0.03--0.14\% while the LR error is 3.7--8.1\%. This asymmetry is expected: the loss minimum is a second-order quantity that is insensitive to perturbations in the sweep points, whereas the minimizing LR is first-order.

\textbf{Five points suffice.} Because loss is second-order in LR, even moderate LR errors translate to negligible loss errors. With $n{=}5$, the worst-case LR error is 4.1\% (at 10.4B tokens), yet the corresponding loss error is only 0.04\%, or ${\sim}0.001$ nats in absolute terms. Throughout this paper, we report losses up-to four decimal places and the smallest architecture differences we act on are ${\sim}0.006$ nats (e.g., SqrtGate vs.\ SharedExp+SqrtGate in \Cref{tab:moe-ablation}). A ${\sim}0.001$-nat fitting uncertainty is thus well below the resolution needed to distinguish any comparison in our experiments. Note that 
our analysis assumes a well-fitting quadratic ($R^2 > 0.99$) and when the fit is poor, we base our conclusions on the observed optimal LR instead.


\subsection{Rethinking Architecture}\label{sec:arch-analysis}

After establishing the scaling laws on the predefined \texttt{Transformer-Next} architecture, we can now revisit architectural design choices to better understand their impact under hypersphere optimization. With HyperP enabling fair comparisons under scalable, near-optimal hyperparameter settings, we study several architecture variants by first identifying the optimal learning rate at the $d=8$ scale and then scaling to larger models using the fitted optima.

\paragraph{Small-scale ablation study.} We first compare architectural variants at the smallest scale ($d{=}8$, 10.4B tokens) before scaling up. On the dense model side, we ablate three attention normalization variants: \textbf{GA QK-Norm} (Gated Attention with QK normalization), \textbf{QK-Norm}, and \textbf{Baseline} (no QK-Norm or GA). On the MoE side, we ablate SqrtGate (\Cref{sec:moe-theory}) and the shared expert (SharedExp) on a sparsity of 8 and a granularity of 8 configuration ($S=8,k=8$), comparing \textbf{SqrtGate}, \textbf{SharedExp}, and \textbf{SharedExp + SqrtGate}.

\begin{figure}[H]
    \centering
    \includegraphics[width=\linewidth]{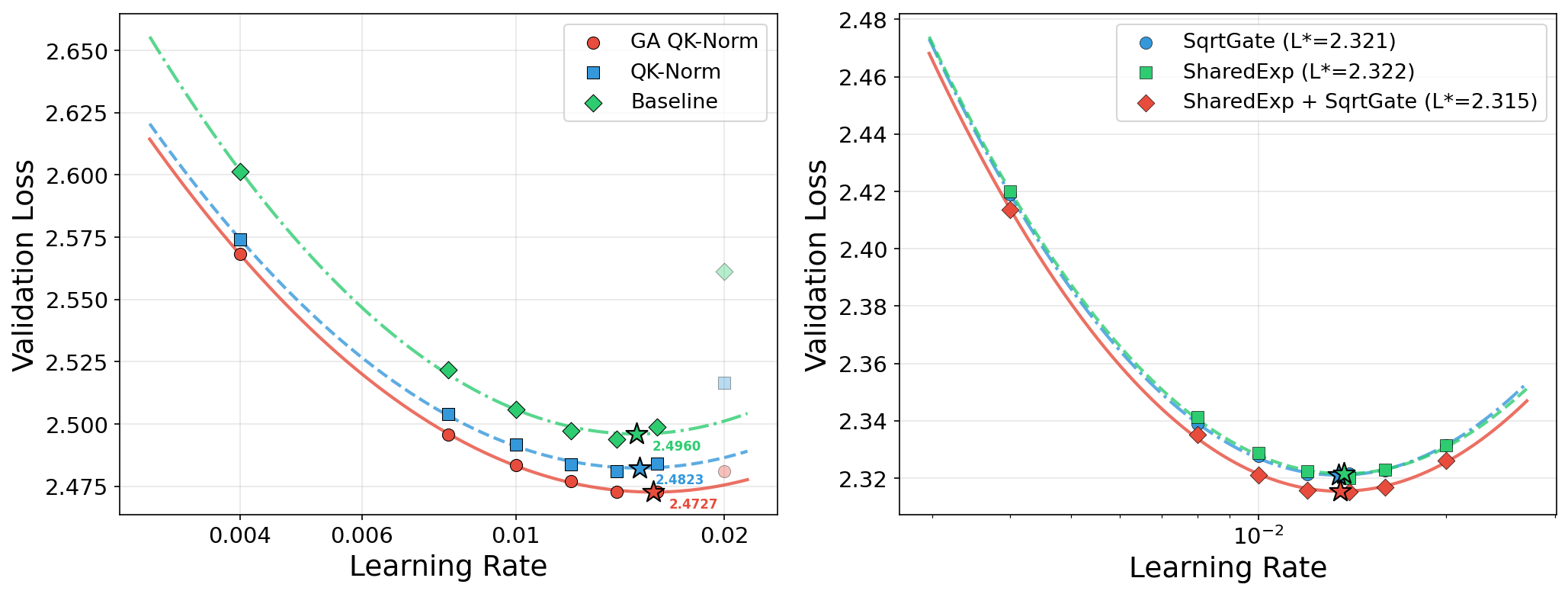}
    \caption{Small-scale LR sweeps at $d{=}8$, 10.4B tokens. Left: Dense attention normalization variants. GA QK-Norm achieves the lowest loss  with a slightly shifted optimal LR. We exclude the LR${=}0.02$ data points for dense models because the large learning rate leads to phase changes that harm fitting goodness. Right: MoE architecture variants. SharedExp + SqrtGate achieves the best loss  while all variants maintain similar optimal LRs ($\eta^* \approx 0.0135$ -- $0.0137$). The exact values are reported in \Cref{tab:qknorm} and \Cref{tab:moe-ablation}.}
    \label{fig:combined-lr-sweep}
\end{figure}

As shown in \Cref{fig:combined-lr-sweep}, GA QK-Norm outperforms QK-Norm by $-0.010$ nats and Baseline by $-0.023$ nats. All three methods have similar optimal LRs ($0.015$--$0.016$), confirming that gated attention with QK normalization directly improves optimization quality without drastically change the LR landscape.
 All three MoE variants share nearly identical optimal LRs ($\eta^* = 0.0135$ -- $0.0137$), indicating that neither SqrtGate nor the shared expert distorts the LR landscape under hypersphere optimization. SqrtGate and the shared expert alone  provide nearly identical improvements. Combining both yields the best performance, suggesting that the two mechanisms address orthogonal aspects: SqrtGate stabilizes the forward signal magnitude across routing granularity (\Cref{prop:moe-sqrt-gating}), while the shared expert provides a consistently activated capacity pathway. 






\begin{figure}[H]
    \centering
    \includegraphics[width=\linewidth]{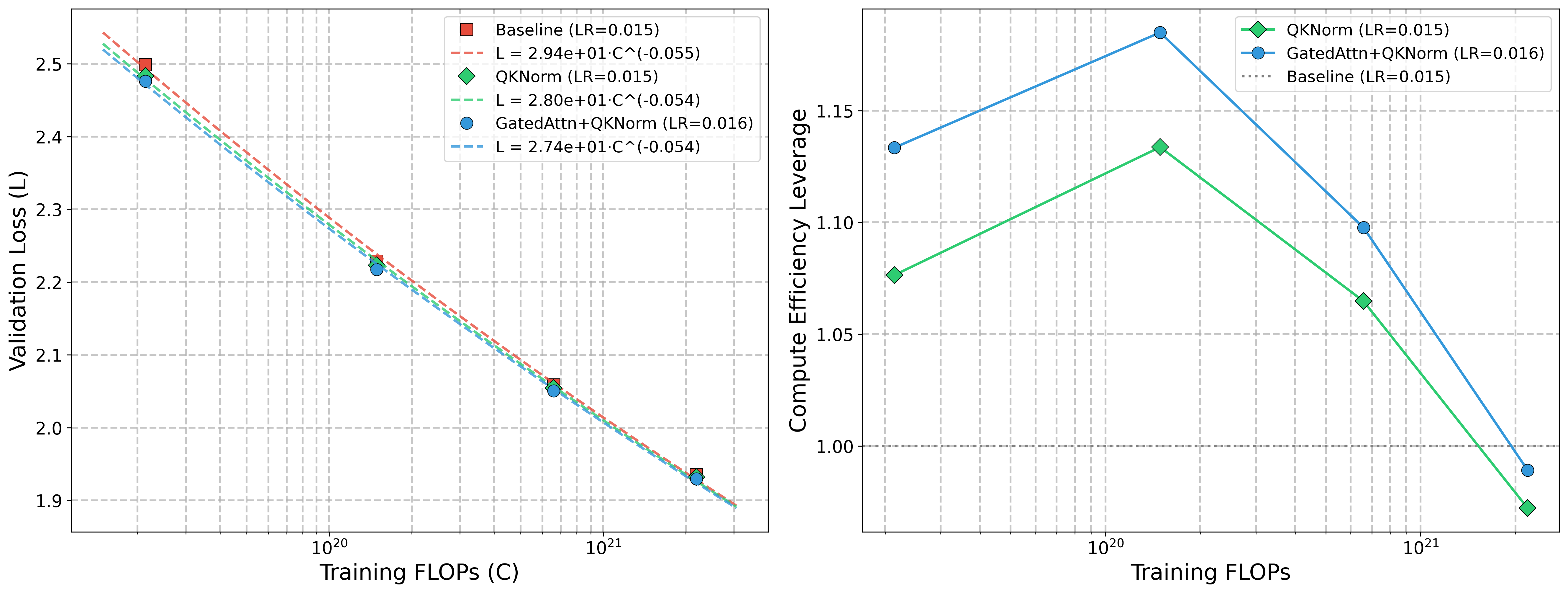}
    \caption{Dense architecture scaling. Left: Loss vs.\ FLOPs with power-law fits $L = A \cdot C^{-b}$. Right: Compute Efficiency Leverage (CEL) over the baseline.}
    \label{fig:ablation-scaling}
\end{figure}

\paragraph{Scaling comparisons.}
In \Cref{fig:ablation-scaling}, we compare Baseline (LR${=}0.015$), QKNorm (LR${=}0.015$), and GatedAttn+QKNorm (LR${=}0.016$) across depths $d \in \{8, 12, 16, 20\}$. Since there are fewer than 5 data points for the fitting, we apply power-law fits without irreducible loss for robust estimates. GatedAttn+QKNorm achieves the best scaling behaviors overall, translating its $-0.023$ nats small-scale advantage into growing compute efficiency leverage that peaks at ${\sim}1.15\times$ at intermediate scale. However, the advantages of both QKNorm and GatedAttn+QKNorm further shrink as the training scale increases, indicating the diminishing performance returns of these architectural choices.

\begin{figure}[H]
    \centering
    \includegraphics[width=\linewidth]{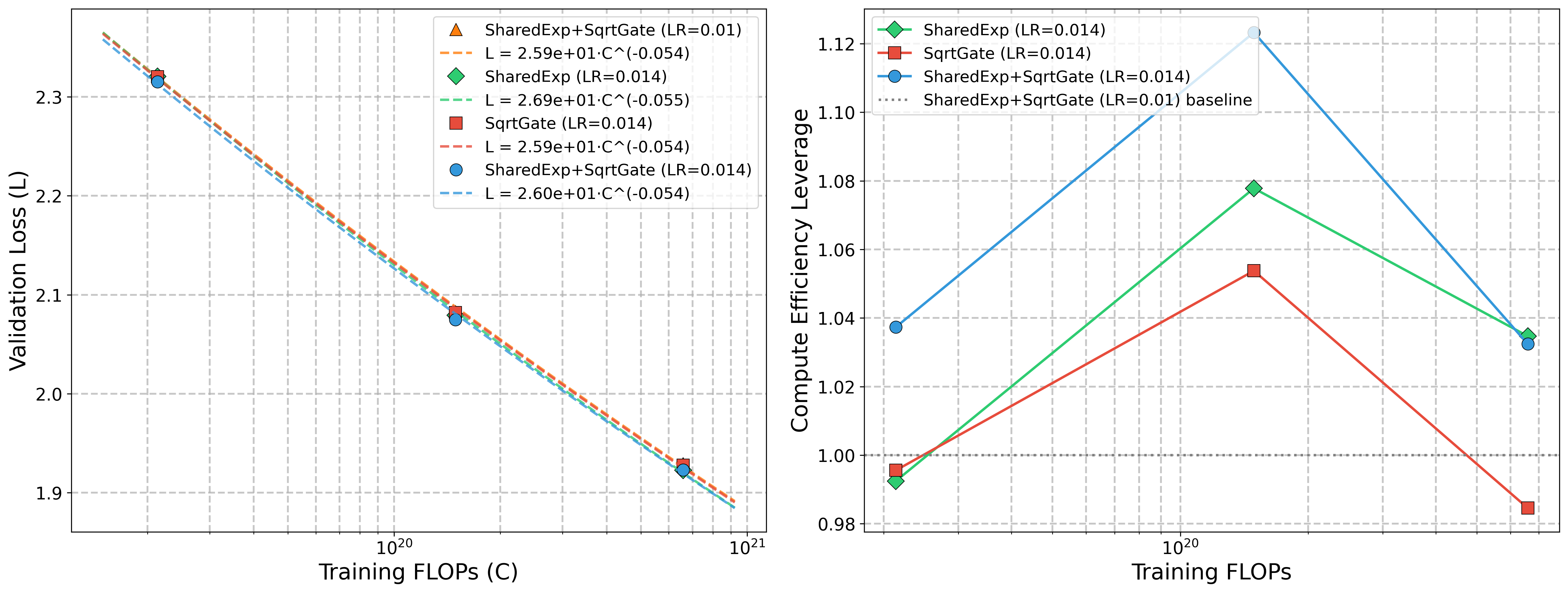}
    \caption{MoE architecture scaling. Left: Loss vs.\ FLOPs with power-law fits. All properly-tuned variants (LR${=}0.014$) outperform the lower-LR baseline. Right: Compute efficiency leverage over the SharedExp+SqrtGate (LR${=}0.01$) baseline. }
    \label{fig:moe-scaling}
\end{figure}

In \Cref{fig:moe-scaling},  we compare SharedExp, SqrtGate, and SharedExp+SqrtGate across depths $d \in \{8, 12, 16, 24\}$ with their fitted optimal LR${=}0.014$. We also have an observed optimal LR baseline for SharedExp+SqrtGate with LR${=}$0.01 to understand how using the fitted optimum affects the CEL when compared to the observed optimum.
All three variants at the properly tuned LR ($0.014$) outperform the SharedExp+SqrtGate baseline at the suboptimal LR ($0.01$), demonstrating that even a modest LR mismatch ($1.4\times$) compounds into meaningful efficiency losses at scale. Among the properly tuned variants, SharedExp+SqrtGate achieves the best overall scaling, confirming that the complementary benefits observed at small scale (\Cref{fig:combined-lr-sweep}) persist with increasing compute. These results underscore that fine-grained LR tuning at small scale, enabled by HyperP's transferable optimal LR, propagates nontrivial compute savings across the entire scaling trajectory.

\begin{figure}[H]
    \centering
    \includegraphics[width=\linewidth]{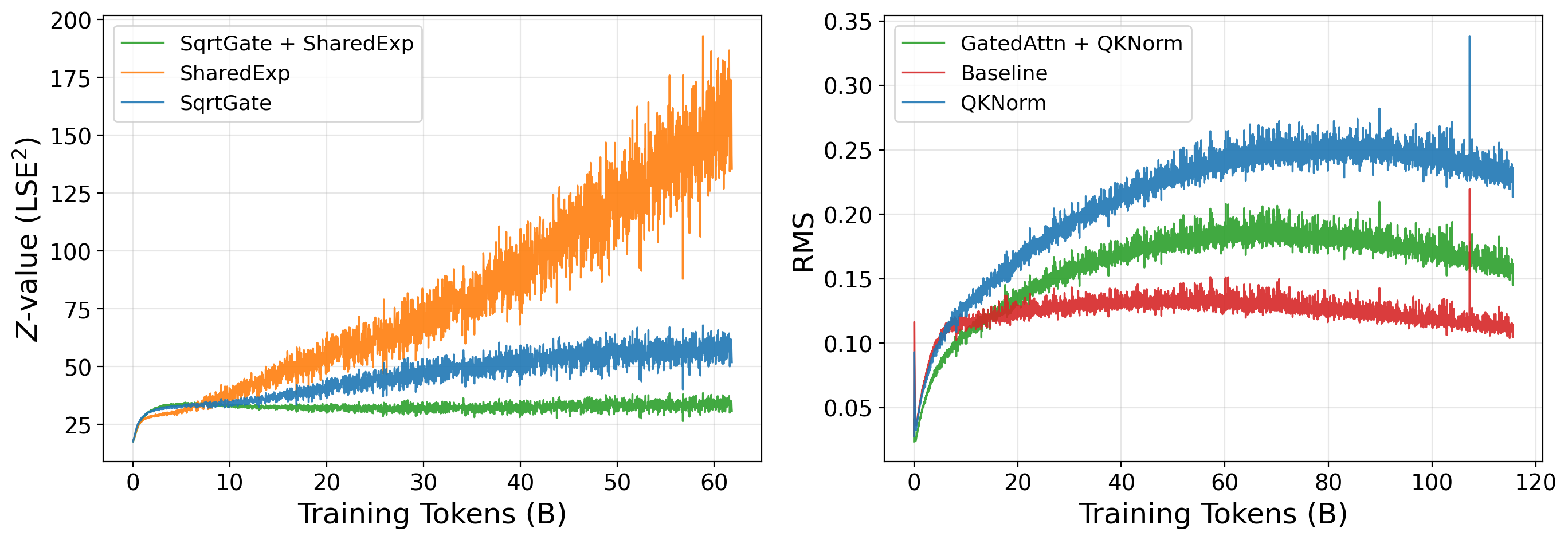}
    \caption{Stability comparison of architecture ablations. Left: Router $Z$-values for MoE variants at $d{=}16$. Router logits are exploding without SqrtGate. Right: MLP output RMS for dense variants at $d{=}20$. QKNorm and Baseline exhibit a large spike at around 110B training tokens, while GatedAttn+QKNorm maintains the most stable RMS throughout training.}
    \label{fig:ablation-stability}
\end{figure}
\paragraph{Stability benefits of architecture choices.}
While the loss improvements from architectural choices narrow with increasing computes, they present significant stability benefits at larger scales. \Cref{fig:ablation-stability} tracks two instability indicators during long training runs at the largest scale in ablation study. For MoE routing (\Cref{fig:ablation-stability}, left), the effect of SqrtGate is dramatic: without SqrtGate (SharedExp only), router $Z$-values grow continuously from ${\sim}25$ to over 190 with frequent spikes, indicating progressive logit explosion. Adding SqrtGate suppresses this entirely: the $Z$-values remain bounded below 40 throughout training, around $5{\times}$ reduction in peak magnitude. For dense architectures (\Cref{fig:ablation-stability}, right), QKNorm alone produces the highest MLP output RMS, while GatedAttn+QKNorm achieves the most stable trend without any spikes. The vanilla Baseline shows lower final RMS but exhibits late-training instability spikes absent from the GatedAttn+QKNorm variant. These results reveal a complementary role for architecture design under HyperP: even when loss differences shrink at scale, the stability margin provided by SqrtGate and GatedAttn becomes increasingly important for reliable large-scale training.

\section{Conclusion}\label{sec:conclusion}

We introduce HyperP (Hypersphere Parameterization), the first framework for transferring a single optimal learning rate across model width, depth, training tokens, and MoE granularity under Frobenius-sphere optimization. We prove that weight decay is a first-order no-op on the Frobenius sphere and that Depth-$\mu$P remains necessary, and empirically identify a data-scaling exponent of $0.32$ matching previous studies on AdamW, suggesting universality across optimizers. For MoE, we propose SqrtGate, a gating mechanism that preserves output RMS across granularities, reducing router $Z$-value peaks by $5\times$. A single base learning rate tuned at $d{=}8$ (208M active parameters) transfers to $d{=}24$ (3.8B active parameters), achieving $1.58\times$ compute efficiency leverage over a strong Muon baseline at $6\times 10^{21}$ FLOPs, with MoE models at $d{=}20$ of 13.3B total parameters reaching $3.38\times$ over dense models. The advantage grows monotonically, suggesting even larger gains than the Muon baseline at frontier scales. HyperP also enables substantially larger auxiliary load-balancing weights and allows models to achieve the best loss and expert balance simultaneously. HyperP further delivers \emph{transferable stability}: all monitored instability indicators are non-increasing with scale under the same transferred hyperparameters, and systematic architecture comparisons reveal that while loss improvements from QK-Norm, Gated Attention, and SqrtGate diminish with scale, their stability benefits become increasingly important for long-horizon training.

\textbf{Limitations.} We assume the Chinchilla law is compute-optimal for our training setup, which in practice needs to be re-fitted per training dataset. The magic data scaling exponent $0.32$ is an empirical observation that lacks a theoretical derivation for a universality guarantee. Extending these scaling laws to other architectures (e.g., hybrid models \cite{ren2025decoderhybriddecoder}, linear recurrent models \cite{lahoti2026mamba030,yang2025gated}) and verifying them at larger scale remains an important future direction. The batch size scaling exponent $0.56$ deviates from the SDE-predicted $0.5$, warranting further theoretical investigation. While our work mainly focuses on the transfer law of the learning rate, better scaling performance may be achieved by extending the analyses to the optimizers' momentum terms.

\section*{Acknowledgement}
We want to thank Kaiyue Wen, Cheng Lu, Songlin Yang and Jingyuan Liu for helpful discussions.

{\small

\bibliography{references}
\bibliographystyle{alpha}
}

\appendix

\section{Proof of Width Transfer under Frobenius Sphere}
\label{app:width-scaling-proof}
\begin{lemma}[Spectral--Frobenius sandwich]
\label{lem:spectral-frobenius-sandwich}
For any matrix $W \in \mathbb{R}^{d_{\mathrm{out}}\times d_{\mathrm{in}}}$,
\[
\|W\|_2 \le \|W\|_F \le \sqrt{r}\,\|W\|_2,
\qquad
r:=\operatorname{rank}(W).
\]
Hence,
\[
\|W\|_2 \le \|W\|_F \le
\sqrt{\min(d_{\mathrm{in}},d_{\mathrm{out}})}\,\|W\|_2.
\]
Moreover, the upper bound is attained if and only if
$r=\min(d_{\mathrm{in}},d_{\mathrm{out}})$ and all nonzero singular values of
$W$ are equal.
\end{lemma}

\begin{proof}
Let $\sigma_1,\dots,\sigma_r$ be the nonzero singular values of $W$. Then
\[
\|W\|_2 = \max_{1\le i\le r}\sigma_i,
\qquad
\|W\|_F = \Bigl(\sum_{i=1}^r \sigma_i^2\Bigr)^{1/2}.
\]
The lower bound $\|W\|_2 \le \|W\|_F$ is immediate. For the upper bound,
\[
\sum_{i=1}^r \sigma_i^2 \le r \max_i \sigma_i^2 = r\|W\|_2^2,
\]
hence
\[
\|W\|_F \le \sqrt{r}\,\|W\|_2.
\]
Equality holds if and only if all nonzero singular values are equal. To also
reach
\[
\|W\|_F = \sqrt{\min(d_{\mathrm{in}},d_{\mathrm{out}})}\,\|W\|_2,
\]
one additionally needs $r=\min(d_{\mathrm{in}},d_{\mathrm{out}})$, i.e.,
full rank.
\end{proof}

We now prove the width-transfer theorem.

\begin{proof}[Proof of Theorem~\ref{thm:width-scaling-hyperball}]
By assumption,
\[
\|W\|_{\mathrm{rms}} = \frac{\|W\|_F}{\sqrt{d_{\mathrm{out}}d_{\mathrm{in}}}}
= \frac{C}{\sqrt{d_{\mathrm{in}}}},
\]
which is equivalent to
\[
\|W\|_F = C\sqrt{d_{\mathrm{out}}}.
\]

Now assume $W$ is approximately isotropic on its input space, so that for
typical inputs $X$,
\[
\|WX\|_2  \le \|W\|_2 \|X\|_2 \approx \frac{\|W\|_F}{\sqrt{\min(d_{\mathrm{in}},d_{\mathrm{out}})}}\,\|X\|_2.
\]
Since $\sqrt{d_{\mathrm{in}}/ \min(d_{\mathrm{in}},d_{\mathrm{out}})} = O(1)$, then
\[
\|Y\|_{\mathrm{rms}}
=
\frac{\|Y\|_2}{\sqrt{d_{\mathrm{out}}}}
=
\frac{\|WX\|_2}{\sqrt{d_{\mathrm{out}}}}
\approx
\frac{\|W\|_F}{\sqrt{d_{\mathrm{in}}}\sqrt{d_{\mathrm{out}}}}\,\|X\|_2.
\]
Substituting $\|W\|_F=C\sqrt{d_{\mathrm{out}}}$ gives
\[
\|Y\|_{\mathrm{rms}}
\approx
\frac{C}{\sqrt{d_{\mathrm{in}}}}\|X\|_2.
\]
Finally, using $\|X\|_{\mathrm{rms}}=\|X\|_2/\sqrt{d_{\mathrm{in}}}$, we obtain
\[
\|Y\|_{\mathrm{rms}}
\approx
C\,\|X\|_{\mathrm{rms}}.
\]
Thus the output rms scale is width-stable, which is exactly the desired
$\mu$P-style width transfer.
\end{proof}

\section{First-Order Form of Frobenius-Sphere Updates}
\label{app:proof_fnorm_tangent_update}

We prove that Frobenius renormalization preserves only the tangent component of an update to first order.

\begin{proof}[Proof of Theorem~\ref{thm:fnorm_tangent_update}]
Vectorize $W$ and $\Delta$ as $w=\mathrm{vec}(W)$ and $\delta=\mathrm{vec}(\Delta)$. Since $\|w\|_2=c_W$, \eqref{eq:fnorm_projected_update_main} becomes
\begin{equation}
w^{+}
=
c_W \frac{w+\delta}{\|w+\delta\|_2}.
\end{equation}
Now expand the denominator:
\begin{align}
\|w+\delta\|_2
&=
\left(\|w\|_2^2 + 2\langle w,\delta\rangle + \|\delta\|_2^2\right)^{1/2} \\
&=
c_W
\left(
1 + 2\frac{\langle w,\delta\rangle}{c_W^2} + \frac{\|\delta\|_2^2}{c_W^2}
\right)^{1/2} \\
&=
c_W
\left(
1 + \frac{\langle w,\delta\rangle}{c_W^2}
\right)
+
O(\|\delta\|_2^2),
\end{align}
where we use $(1+z)^{1/2}=1+\frac12 z + O(z^2)$.

Therefore,
\begin{align}
w^{+}
&=
(w+\delta)
\left(
1 + \frac{\langle w,\delta\rangle}{c_W^2}
\right)^{-1}
+
O(\|\delta\|_2^2) \\
&=
(w+\delta)
\left(
1 - \frac{\langle w,\delta\rangle}{c_W^2}
\right)
+
O(\|\delta\|_2^2) \\
&=
w
+
\delta
-
\frac{\langle w,\delta\rangle}{c_W^2}w
+
O(\|\delta\|_2^2).
\end{align}
Hence
\begin{equation}
w^{+}-w
=
\delta
-
\frac{\langle w,\delta\rangle}{c_W^2}w
+
O(\|\delta\|_2^2).
\end{equation}
Returning to matrix form gives
\begin{equation}
W^{+}-W
=
\Delta
-
\frac{\langle \Delta,W\rangle_F}{\|W\|_F^2}W
+
O(\|\Delta\|_F^2)
=
\Pi_T(\Delta)+O(\|\Delta\|_F^2),
\end{equation}
which proves \eqref{eq:fnorm_tangent_expansion_main}.
\end{proof}

\begin{proof}[Proof of Corollary~\ref{cor:wd_noop}]
By linearity of $\Pi_T$,
\begin{equation}
\Pi_T(G+\lambda W)=\Pi_T(G)+\lambda \Pi_T(W).
\end{equation}
But
\begin{equation}
\Pi_T(W)
=
W-\frac{\langle W,W\rangle_F}{\|W\|_F^2}W
=
W-W=0.
\end{equation}
Applying Theorem~\ref{thm:fnorm_tangent_update} with $\Delta=-\eta(G+\lambda W)$ gives the result.
\end{proof}

\section{Depth Scaling under Frobenius-Sphere Optimization}
\label{app:proof_depth_scaling_fnorm}

We first derive the first-order decomposition of the network perturbation, then analyze residual networks with and without update normalization, and finally extend the argument to post-norm blocks by computing the LayerNorm Jacobian explicitly.

\subsection{First-order decomposition of the network perturbation}

Let
$
F(x;W_1,\dots,W_L)
$
denote the network output as a function of all layer parameters. For perturbations $\Delta W_1,\dots,\Delta W_L$, first-order multivariate Taylor expansion yields
\begin{equation}
F(x;W+\Delta W)-F(x;W)
=
\sum_{l=1}^L \frac{\partial F}{\partial W_l}\Delta W_l
+
O\!\left(\sum_{i,j}\|\Delta W_i\|_F \|\Delta W_j\|_F\right).
\label{eq:app_taylor_total}
\end{equation}
Thus the first-order total perturbation is additive over layers.

\subsection{Residual network without update normalization}

Consider
\begin{equation}
x_{l+1}=x_l+\alpha_L f_l(x_l;W_l).
\label{eq:app_prenorm_block}
\end{equation}
Perturbing both the hidden state and the weights gives, to first order,
\begin{equation}
\Delta x_{l+1}
=
\Delta x_l
+
\alpha_L \frac{\partial f_l}{\partial x_l}\Delta x_l
+
\alpha_L \frac{\partial f_l}{\partial W_l}\Delta W_l.
\label{eq:app_linearized_residual}
\end{equation}
Define
$
A_l = I + \alpha_L J_{f,x}^{(l)}$ and $b_l = \alpha_L J_{f,W}^{(l)}\Delta W_l.
$
Then $\Delta x_{l+1}=A_l \Delta x_l + b_l$.
Unrolling this recursion yields
\begin{equation}
\Delta x_L
=
\sum_{l=1}^{L}
\left(
\prod_{k=l+1}^{L-1} A_k
\right)b_l.
\label{eq:app_unrolled_sum}
\end{equation}
Since $\|J_{f,x}^{(l)}\|_{F}=O(1)$ with depth, for sufficiently small $\alpha_L$, each $A_l$ has an operator norm $O(1)$. Hence the downstream transport factors in \eqref{eq:app_unrolled_sum} contribute only constant-order factors at the level of depth exponents, and it suffices to track the scaling of $b_l$.

If only the weights are normalized and the weight updates scale with the magnitude of gradients, then by Theorem~\ref{thm:fnorm_tangent_update},
$\|\Delta W_l\|_F = O(\eta_l \|G_l\|_F)$.
Under the stable-depth assumption for residual networks, the layerwise gradient satisfies
$\|G_l\|_F = O(\alpha_L)$,
since differentiating \eqref{eq:app_prenorm_block} with respect to $W_l$ introduces exactly one factor of $\alpha_L$, while the upstream signal is $O(1)$ by assumption. Therefore
$\|\Delta W_l\|_F = O(\eta_l \alpha_L)$.
Since $\|J_{f,W}^{(l)}\|_F=O(1)$,
\begin{equation}
\|b_l\|_F
=
O(\alpha_L \|\Delta W_l\|_F)
=
O(\eta_l \alpha_L^2).
\end{equation}
By the Triangle Inequality, summing over $L$ layers gives
$\|\Delta x_L\|_F = O(L \eta_l \alpha_L^2)$.
We consider two cases where alpha is sufficiently small to satisfy our assumptions: If $\alpha_L=L^{-1/2}$, this reduces to $\|\Delta x_L\|_F = O(\eta_l)$,
so a depth-independent learning rate $\eta_l=O(1)$ yields an $O(1)$ first-order function perturbation; If $\alpha_L=L^{-1}$, one obtains
$\|\Delta x_L\|_F = O(L^{-1}\eta_l)$,
which requires $\eta_l=O(L)$.

\subsection{Residual network with update normalization}

Assume the raw update is normalized before the Frobenius-sphere projection:
\begin{equation}
\widehat G_l = c_G \frac{G_l}{\|G_l\|_F},
\qquad
\widetilde W_l = W_l - \eta_l \widehat G_l,
\qquad
W_l^{+} = c_W \frac{\widetilde W_l}{\|\widetilde W_l\|_F}.
\end{equation}
By Theorem~\ref{thm:fnorm_tangent_update},
\begin{equation}
\Delta W_l
=
W_l^{+}-W_l
=
-\eta_l \Pi_T(\widehat G_l) + O(\eta_l^2).
\end{equation}
Because \(\|\widehat G_l\|_F=c_G=O(1)\), we have
$
\|\Delta W_l\|_F = O(\eta_l).
$
Thus the linearized residual contribution at layer \(l\) scales as
\begin{equation}
\|b_l\|_F
=
O(\alpha_L \|\Delta W_l\|_F)
=
O(\alpha_L \eta_l),
\end{equation}
and summing over depth gives
\begin{equation}
\|\Delta x_L\|_F = O(L\alpha_L\eta_l).
\end{equation}
Therefore, preserving an \(O(1)\) first-order function perturbation requires
\begin{equation}
\eta_l = O\!\left(\frac{1}{L\alpha_L}\right).
\end{equation}
In particular, if $\alpha_L=L^{-1/2}$, this gives $\eta_l = O(L^{-1/2})$,
while if $\alpha_L=L^{-1}$, it gives $\eta_l = O(1)$.

\subsection{Post-norm residual block}

We now consider
\begin{equation}
x_{l+1}
=
\mathrm{LN}\!\left(x_l+\alpha_L f_l(x_l;W_l)\right).
\label{eq:app_postnorm_block}
\end{equation}
Let
$u = x + \alpha_L f(x;W)$, $\mu = \frac{1}{d}\mathbf 1^\top u$, $v = u-\mu \mathbf 1$, $\sigma = \sqrt{\frac{1}{d}\|v\|^2+\varepsilon}$.
Ignoring learned gain and bias, LayerNorm is $\mathrm{LN}(u)=v/\sigma$.
Let $P = I - \frac{1}{d}\mathbf 1 \mathbf 1^\top$. Since $v=Pu$, we have $dv = P\,du$. Moreover,
\begin{equation}
d\sigma
=
\frac{1}{2\sigma}d\!\left(\frac{1}{d}\|v\|^2+\varepsilon\right)
=
\frac{1}{\sigma d} v^\top dv
=
\frac{1}{\sigma d}v^\top du,
\end{equation}
where we use $Pv=v$. Now
\begin{equation}
d\,\mathrm{LN}(u)
=
d\left(\frac{v}{\sigma}\right)
=
\frac{1}{\sigma}dv - \frac{v}{\sigma^2}d\sigma
=
\frac{1}{\sigma}P\,du
-
\frac{1}{\sigma^3 d} vv^\top du.
\end{equation}
Hence the LayerNorm Jacobian is
\begin{equation}
J_{\mathrm{LN}}(u)
=
\frac{1}{\sigma}P - \frac{1}{\sigma^3 d}vv^\top
=
\frac{1}{\sigma}\left(P-\frac{1}{d\sigma^2}vv^\top\right).
\label{eq:app_ln_jacobian}
\end{equation}

Now differentiate \eqref{eq:app_postnorm_block}. By the chain rule,
\begin{equation}
\frac{\partial x_{l+1}}{\partial W_l}
=
J_{\mathrm{LN}}(u_l)\,\alpha_L \frac{\partial f_l}{\partial W_l},
\end{equation}
and similarly
$
\frac{\partial x_{l+1}}{\partial x_l}
=
J_{\mathrm{LN}}(u_l)\left(I+\alpha_L \frac{\partial f_l}{\partial x_l}\right).
$
Since $\sigma_l$ is $O(1)$ with depth under post-norm, then by \eqref{eq:app_ln_jacobian},
$\|J_{\mathrm{LN}}(u_l)\|_{\mathrm{op}} = O(1)$.
Therefore, the local weight sensitivity has the same depth scaler $\alpha_L$ as in the pre-norm residual block. Consequently, the same scaling argument as above applies: with update normalization, one obtains
$\|\Delta x_L\| = O(L\alpha_L \eta_l)$,
and hence $\eta_l = O\!\left(\frac{1}{L\alpha_L}\right)$.


\section{Detailed Experimental Results}
\label{app:detailed-tables}

This section provides the exact numerical values for the figures presented in \Cref{sec:data}, \Cref{sec:experiments} and \Cref{sec:analysis}.

\begin{table}[H]
\centering
\caption{Optimal LR vs.\ training token budget under fine-grid sweeping with quadratic fitting.}
\label{tab:data-scaling}
\vspace{0.1cm}
\begin{tabular}{ccc}
\toprule
\textbf{Training Tokens} & \textbf{Fitted $\eta^*$} & \textbf{Fitted Min Loss} \\
\midrule
10.4B  & 0.01515 & 2.4741 \\
20.8B  & 0.01208 & 2.4189 \\
41.6B  & 0.00958 & 2.3773 \\
83.2B  & 0.00772 & 2.3456 \\
166.4B & 0.00635 & 2.3214 \\
\bottomrule
\end{tabular}
\end{table}

\begin{table}[H]
\centering
\caption{Validation loss vs.\ LR across model depth at a fixed token budget of 10.4B without Depth-$\mu$P.}
\label{tab:depth-scaling}
\vspace{0.1cm}
\resizebox{\linewidth}{!}{
\begin{tabular}{ccccccccccccc}
\toprule
\textbf{Depth ($d$)} & \textbf{Params} & $\eta{=}0.002$ & $\eta{=}0.004$ & $\eta{=}0.006$ & $\eta{=}0.008$ & $\eta{=}0.010$ & $\eta{=}0.012$ & $\eta{=}0.014$ & $\eta{=}0.016$ & $\eta{=}0.018$ & $\eta{=}0.020$ & \textbf{Optimal} $\eta$ \\
\midrule
8  & 208M & 2.684 & 2.569 & 2.521 & 2.498 & 2.485 & 2.474 & 2.470 & \textbf{2.469} & 2.473 & 2.492 & 0.016 \\
12 & 570M & 2.523 & 2.405 & 2.355 & 2.328 & 2.315 & \textbf{2.308} & 2.309 & 2.319 & 2.351 & 2.386 & 0.012 \\
16 & 1.24B & 2.426 & 2.307 & 2.256 & 2.230 & \textbf{2.220} & 2.225 & 2.251 & 2.288 & 2.299 & 2.315 & 0.010 \\
20 & 2.31B & 2.354 & 2.235 & 2.189 & \textbf{2.166} & 2.169 & 2.191 & 2.218 & 2.235 & 2.246 & 2.264 & 0.008 \\
24 & 3.90B & 2.300 & 2.183 & 2.140 & \textbf{2.126} & 2.142 & 2.165 & 2.184 & 2.196 & 2.212 & 2.221 & 0.008 \\
\bottomrule
\end{tabular}
}
\end{table}

\begin{table}[H]
\centering
\caption{Validation loss vs.\ LR across model depth at a fixed token budget of 10.4B with Depth-$\mu$P.}
\label{tab:depth-scaling-mup}
\vspace{0.1cm}
\resizebox{\linewidth}{!}{
\begin{tabular}{ccccccccccccc}
\toprule
\textbf{Depth ($d$)} & \textbf{Params} & $\eta{=}0.002$ & $\eta{=}0.004$ & $\eta{=}0.006$ & $\eta{=}0.008$ & $\eta{=}0.010$ & $\eta{=}0.012$ & $\eta{=}0.014$ & $\eta{=}0.016$ & $\eta{=}0.018$ & $\eta{=}0.020$ & \textbf{Optimal} $\eta$ \\
\midrule
8  & 208M & 2.682 & 2.568 & 2.520 & 2.496 & 2.484 & 2.476 & \textbf{2.473} & 2.474 & 2.477 & 2.479 & 0.014 \\
12 & 570M & 2.568 & 2.437 & 2.377 & 2.347 & 2.331 & 2.319 & 2.316 & \textbf{2.315} & 2.317 & 2.321 & 0.016 \\
16 & 1.24B & 2.495 & 2.359 & 2.297 & 2.264 & 2.245 & 2.235 & 2.229 & \textbf{2.225} & 2.225 & 2.234 & 0.016 \\
20 & 2.31B & 2.445 & 2.309 & 2.246 & 2.211 & 2.188 & 2.177 & 2.171 & \textbf{2.169} & 2.172 & 2.188 & 0.016 \\
24 & 3.90B & 2.413 & 2.272 & 2.208 & 2.172 & 2.150 & 2.137 & \textbf{2.132} & 2.132 & 2.137 & 2.152 & 0.014 \\
\bottomrule
\end{tabular}
}
\end{table}

\begin{table}[H]
\centering
\caption{Optimal LR and loss vs.\ model depth at a fixed training token of 10.4B, comparing with and without Depth-$\mu$P. Depth-$\mu$P transfers the optimal LR across parameter size.}
\label{tab:depth-mup}
\vspace{0.1cm}
\begin{tabular}{ccccc}
\toprule
\textbf{Depth ($d$)} & \textbf{w/ Depth-$\mu$P $\eta^*$} & \textbf{w/ Depth-$\mu$P Loss} & \textbf{w/o Depth-$\mu$P $\eta^*$} & \textbf{w/o Depth-$\mu$P Loss} \\
\midrule
8  & 0.014 & 2.4734 & 0.016 & 2.4693 \\
12 & 0.016 & 2.3150 & 0.012 & 2.3079 \\
16 & 0.016 & 2.2250 & 0.010 & 2.2196 \\
20 & 0.016 & 2.1690 & 0.008 & 2.1656 \\
24 & 0.014 & 2.1320 & 0.008 & 2.1263 \\
\bottomrule
\end{tabular}
\end{table}

\begin{table}[H]
\centering
\caption{Optimal LR vs.\ batch size for dense models. The minimum loss is remarkably stable (within 0.004 nats), indicating all tested batch sizes are below the critical batch size.}
\label{tab:bsz-scaling}
\vspace{0.1cm}
\begin{tabular}{ccc}
\toprule
\textbf{Batch Size} & \textbf{Fitted $\eta^*$} & \textbf{Fitted Min Loss} \\
\midrule
256K & 0.00504 & 2.4711 \\
512K & 0.00706 & 2.4697 \\
1M   & 0.01056 & 2.4700 \\
2M   & 0.01562 & 2.4741 \\
\bottomrule
\end{tabular}
\end{table}

\begin{table}[H]
\centering
\caption{Validation loss vs.\ LR across auxiliary loss weights. The optimal LR ($\eta^* = 0.012$) and achievable loss are stable across a $100\times$ range of $\gamma$.}
\label{tab:auxloss}
\vspace{0.1cm}
\begin{tabular}{ccccccc}
\toprule
$\gamma$ & $\eta{=}0.004$ & $\eta{=}0.008$ & $\eta{=}0.01$ & $\eta{=}0.012$ & $\eta{=}0.02$ & \textbf{Best Loss} \\
\midrule
$10^{-3}$ & 2.427 & 2.350 & 2.340 & \textbf{2.334} & 2.349 & 2.334 \\
$10^{-2}$ & 2.431 & 2.354 & 2.340 & \textbf{2.336} & 2.346 & 2.336 \\
$10^{-1}$ & 2.427 & 2.350 & 2.337 & \textbf{2.332} & 2.346 & 2.332 \\
\bottomrule
\end{tabular}
\end{table}

\begin{table}[H]
\centering
\caption{Optimal LR and loss vs.\ MoE sparsity. The LR varies mildly (0.012--0.016) across a $32\times$ range.}
\label{tab:sparsity}
\vspace{0.1cm}
\begin{tabular}{ccc}
\toprule
\textbf{Sparsity ($S$)} & \textbf{Fitted $\eta^*$} & \textbf{Fitted Min Loss} \\
\midrule
1  & 0.0163 & 2.4766 \\
2  & 0.0162 & 2.4236 \\
4  & 0.0145 & 2.3705 \\
8  & 0.0139 & 2.3262 \\
16 & 0.0124 & 2.2861 \\
32 & 0.0115 & 2.2529 \\
\bottomrule
\end{tabular}
\end{table}

\begin{table}[H]
\centering
\caption{Optimal LR and loss across top-$k$ values, with and without SqrtGate.}
\label{tab:topk}
\vspace{0.1cm}
\begin{tabular}{ccccc}
\toprule
\textbf{Top-$k$} & \textbf{w/o SqrtGate $\eta^*$} & \textbf{w/o SqrtGate Loss} & \textbf{w/ SqrtGate $\eta^*$} & \textbf{w/ SqrtGate Loss} \\
\midrule
2  & 0.0140 & 2.4306 & 0.0139 & 2.4131 \\
4  & 0.0132 & 2.3263 & 0.0139 & 2.3262 \\
8  & 0.0137 & 2.3220 & 0.0135 & 2.3156 \\
16 & 0.0126 & 2.3178 & 0.0129 & 2.3111 \\
32 & 0.0127 & 2.3186 & 0.0131 & 2.3096 \\
64 & 0.0122 & 2.3244 & 0.0128 & 2.3154 \\
\bottomrule
\end{tabular}
\end{table}

\begin{table}[H]
\centering
\caption{Validation loss vs.\ FLOPs. MuonH + HyperP increasingly outperforms both alternatives at larger scale.}
\label{tab:flops-scaling}
\vspace{0.1cm}
\begin{tabular}{ccccc}
\toprule
\textbf{Depth} & \textbf{FLOPs} & \textbf{Muon} & \textbf{MuonH+HyperP} & \textbf{MuonH } \\
\midrule
8  & $2.14\times10^{19}$ & 2.4777 & 2.4804 & 2.4845 \\
12 & $1.49\times10^{20}$ & 2.2257 & 2.2192 & 2.2099 \\
16 & $6.59\times10^{20}$ & 2.0671 & 2.0526 & 2.0500 \\
20 & $2.19\times10^{21}$ & 1.9591 & 1.9311 & 1.9558 \\
24 & $5.96\times10^{21}$ & 1.8785 & 1.8365 & 1.9015 \\
\bottomrule
\end{tabular}
\end{table}

\begin{table}[H]
\centering
\caption{Compute efficiency leverage over Muon. MuonH + HyperP's advantage grows monotonically with scale.}
\label{tab:leverage}
\vspace{0.1cm}
\begin{tabular}{cccc}
\toprule
\textbf{Depth} & \textbf{FLOPs} & \textbf{MuonH+HyperP} & \textbf{MuonH } \\
\midrule
8  & $2.14\times10^{19}$ & $0.99\times$ & $0.96\times$ \\
12 & $1.49\times10^{20}$ & $1.04\times$ & $1.19\times$ \\
16 & $6.59\times10^{20}$ & $1.16\times$ & $1.17\times$ \\
20 & $2.19\times10^{21}$ & $1.35\times$ & $0.99\times$ \\
24 & $5.96\times10^{21}$ & $1.58\times$ & $0.70\times$ \\
\bottomrule
\end{tabular}
\end{table}

\begin{table}[H]
\centering
\caption{Mean relative error (\%) of optimal LR and optimal loss estimates as a function of the number of sweep points $n$.}
\label{tab:sensitivity}
\vspace{0.1cm}
\begin{tabular}{cccccccccccc}
\toprule
 & \multicolumn{5}{c}{\textbf{LR Relative Error (\%)}} & \multicolumn{5}{c}{\textbf{Loss Relative Error (\%)}} \\
\cmidrule(lr){2-6} \cmidrule(lr){7-11}
\textbf{Tokens} & $n{=}3$ & $n{=}4$ & $n{=}5$ & $n{=}6$ & $n{=}7$ & $n{=}3$ & $n{=}4$ & $n{=}5$ & $n{=}6$ & $n{=}7$ \\
\midrule
10.4B  & 5.87 & 5.08 & 4.09 & 2.88 & 1.55 & 0.07 & 0.05 & 0.04 & 0.03 & 0.01 \\
20.8B  & 3.68 & 2.01 & 1.46 & 1.00 & 0.53 & 0.04 & 0.02 & 0.01 & 0.01 & 0.01 \\
41.6B  & 4.27 & 2.55 & 1.58 & 0.95 & 0.46 & 0.03 & 0.02 & 0.01 & 0.01 & 0.01 \\
83.2B  & 5.43 & 3.06 & 1.88 & 1.12 & 0.54 & 0.05 & 0.02 & 0.02 & 0.01 & 0.01 \\
166.4B & 8.07 & 1.97 & 0.88 & 0.48 & 0.28 & 0.14 & 0.02 & 0.01 & 0.01 & 0.00 \\
\bottomrule
\end{tabular}
\end{table}

\begin{table}[H]
\centering
\caption{QK-Norm ablation at $d{=}8$. GA QK-Norm achieves the best loss while maintaining a similar optimal LR.}
\label{tab:qknorm}
\vspace{0.1cm}
\begin{tabular}{lcc}
\toprule
\textbf{Method} & \textbf{Fitted $\eta^*$} & \textbf{Min Loss } \\
\midrule
GA QK-Norm & 0.0158 & \textbf{2.4727} \\
QK-Norm    & 0.0151 & 2.4823 \\
Baseline   & 0.0149 & 2.4960 \\
\bottomrule
\end{tabular}
\end{table}

\begin{table}[H]
\centering
\caption{MoE architecture ablation at $d{=}8$, $S{=}8$, $k{=}8$, 10.4B tokens. SqrtGate and the shared expert provide complementary gains.}
\label{tab:moe-ablation}
\vspace{0.1cm}
\begin{tabular}{lccc}
\toprule
\textbf{Method} & \textbf{Fitted $\eta^*$} & \textbf{Min Loss} & \textbf{$\Delta$ vs.\ Best} \\
\midrule
SharedExp + SqrtGate & 0.0135 & \textbf{2.3154} & --- \\
SqrtGate             & 0.0135 & 2.3210 & $+0.006$ \\
SharedExp            & 0.0137 & 2.3215 & $+0.006$ \\
\bottomrule
\end{tabular}
\end{table}

\end{document}